\documentclass{article}

\usepackage{microtype}
\usepackage{graphicx}
\usepackage[normalem]{ulem}
\usepackage{booktabs} 
\usepackage[many]{tcolorbox}
\usetikzlibrary{shadows}

\newtcolorbox{shadedbox}{
  drop shadow southeast,
  breakable,
  enhanced jigsaw,
  colback=white,
}
\usepackage{hyperref}
\DeclareMathOperator{\tr}{tr}

\usepackage[accepted]{icml2020}

\icmltitlerunning{Multi-Agent Determinantal Q-Learning}

\usepackage{subcaption}
\usepackage{amsmath,amsthm,amscd,amsbsy,amssymb,latexsym,url,bm}

\usepackage{cuted}
\usepackage{flushend}

\usepackage{amsfonts}       
\usepackage{nicefrac}       
\usepackage{microtype}      

\usepackage{helvet}
\usepackage{courier}
\usepackage{multirow}
\usepackage{graphicx}
\usepackage{enumitem}
\usepackage{ragged2e}
\usepackage{epsfig}
\usepackage{dsfont}

\theoremstyle{plain}
\newtheorem{proposition}{Proposition}
\newtheorem{theorem}{Theorem}
\newtheorem{lemma}{Lemma}

\newtheorem{assumption}{Assumption}
\newtheorem{definition}{Definition}

\expandafter\def\expandafter\normalsize\expandafter{%
    \normalsize
    \setlength\abovedisplayskip{6pt}
    \setlength\belowdisplayskip{6pt}
    \setlength\abovedisplayshortskip{6pt}
    \setlength\belowdisplayshortskip{6pt}
}
\DeclareMathOperator{\diag}{diag}

\DeclareRobustCommand{\frac}[3][0pt]{{\begingroup\hspace{#1}#2\hspace{#1}\endgroup\over\hspace{#1}#3\hspace{#1}}}

\newcommand{\bs}{\boldsymbol}

\begin{document}

\twocolumn[
\icmltitle{Multi-Agent Determinantal Q-Learning}



\icmlsetsymbol{equal}{*}

\begin{icmlauthorlist}
\icmlauthor{Yaodong Yang}{equal,huawei,ucl}
\icmlauthor{Ying Wen}{equal,huawei,ucl}
\icmlauthor{Liheng Chen}{sj}
\icmlauthor{Jun Wang}{ucl}
\icmlauthor{Kun Shao}{huawei}
\icmlauthor{David Mguni}{huawei}
\icmlauthor{Weinan Zhang}{sj}

\end{icmlauthorlist}

\icmlaffiliation{ucl}{University College London.}
\icmlaffiliation{huawei}{Huawei Technology R\&D UK.}
\icmlaffiliation{sj}{Shanghai Jiaotong University}
\icmlaffiliation{huawei}{Huawei Technology R\&D UK.}

\icmlcorrespondingauthor{Yaodong Yang}{yaodong.yang@huawei.com}

\icmlkeywords{Machine Learning, ICML}

\vskip 0.3in
]

\printAffiliationsAndNotice{\icmlEqualContribution} 

\begin{abstract}

Centralized training with decentralized execution has become an important paradigm in multi-agent learning. Though practical, current methods rely on restrictive assumptions to decompose the centralized  value function across agents for execution. In this paper, we eliminate this restriction by proposing multi-agent determinantal Q-learning. Our method is established on Q-DPP, an extension of determinantal point process (DPP) with  partition-matroid constraint to multi-agent setting. Q-DPP promotes agents to acquire diverse behavioral models; this allows  a natural factorization of the joint Q-functions with no need for \emph{a priori} structural constraints on the value function or special network architectures. We demonstrate that Q-DPP generalizes major solutions including VDN, QMIX, and QTRAN on decentralizable cooperative tasks. To efficiently draw samples  from Q-DPP, we adopt an existing linear-time sampler with theoretical approximation guarantee. The sampler also benefits exploration by coordinating agents to cover orthogonal directions in the state space during multi-agent training. We evaluate our algorithm on various cooperative benchmarks; its effectiveness has been demonstrated when compared with the state-of-the-art. 

\end{abstract}

\section{Introduction}
Multi-agent reinforcement learning (MARL) methods hold great potential to solve a variety of real-world problems, such as mastering multi-player video games \cite{peng2017multiagent}, dispatching taxi orders  \cite{li2019efficient}, and studying population dynamics \cite{yang2018study}.   
In this work, we consider  the multi-agent cooperative  setting \cite{panait2005cooperative} where a team of agents collaborate to achieve one common goal in a partially observed environment.

A full spectrum of MARL algorithms  has been developed to solve cooperative tasks \cite{panait2005cooperative}; the two endpoints of the spectrum are independent  and centralized  learning (see Fig.~\ref{fig:algo_types}). 
Independent learning (IL) \cite{tan1993multi} merely treats   other agents' influence to the system  as part of the environment. 
The learning agent not only faces a non-stationary environment, but also suffers from    \emph{spurious} rewards  \cite{sunehag2017value}.  
Centralized learning (CL), in the other extreme, treats a multi-agent problem as a single-agent problem  despite the fact that many real-world applications require local autonomy. 
Importantly, the CL approaches exhibit  combinatorial complexity and can hardly scale to more than tens of agents \cite{yang2019alpha}.

Another paradigm typically considered is a hybrid of centralized training and decentralized execution (CTDE)  \cite{oliehoek2008optimal}. 
For value-based approaches in the framework of CTDE, 
a fundamental challenge is how to correctly decompose the centralized value function among agents for decentralized execution.  
For a cooperative task to be deemed decentralizable, it is required that local maxima on the value function per every agent should amount to the global maximum on the joint value function.  
In enforcing such a condition, current state-of-the-art methods rely on restrictive  structural constraints and/or network architectures.
For instance, 
Value Decomposition Network (VDN) \cite{sunehag2017value} and Factorized-Q \cite{zhou2019factorized} propose to directly  factorize the joint value function into a summation of individual value functions.
QMIX \cite{rashid2018qmix}  augments the summation to be non-linear aggregations, while maintaining a monotonic  relationship between centralized and individual value functions. 
QTRAN \cite{son2019qtran} introduces a refined learning objective on top of QMIX along with specific network designs. 

Unsurprisingly, the structural constraints put forward by VDN / QMIX / QTRAN   inhibit the representational power of the centralized value function   \cite{son2019qtran}; as a result,  
the class of decentralizable cooperative tasks these methods can tackle is limited.
For example, poor empirical results of QTRAN have been reported on multiple multi-agent cooperative benchmarks  \cite{mahajan2019maven}. 

\begin{figure}[t!]
    \centering
            \vspace{-8pt}
    \includegraphics[width=0.42\textwidth]{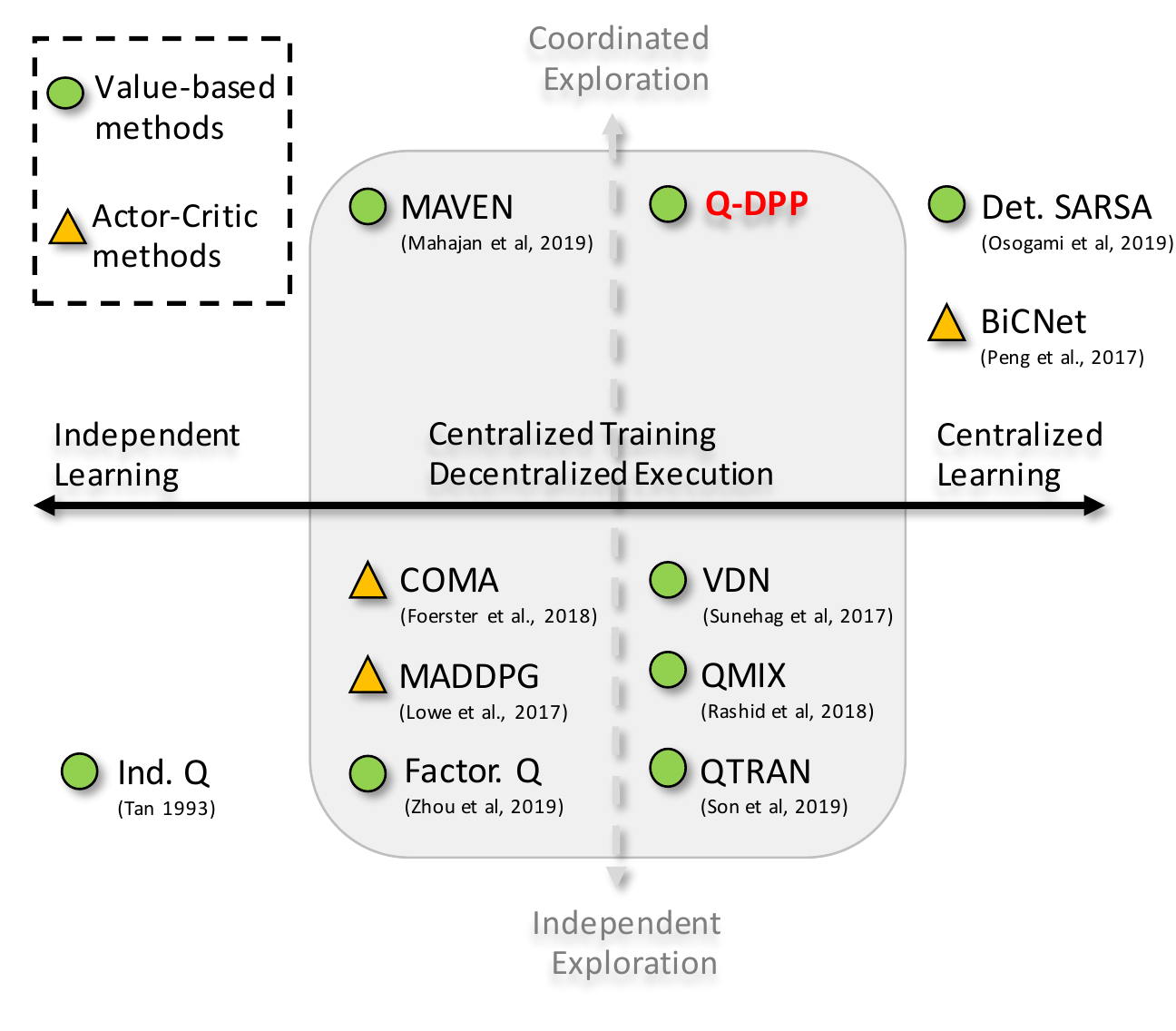}
        \vspace{-15pt}
    \caption{Spectrum of MARL methods on cooperative tasks.
    }
    \label{fig:algo_types}
    \vspace{-0pt}
\end{figure}

Apart from the aforementioned problems, structural constraints also hinder efficient explorations when applied to value function decomposition.
In fact, since agents are treated independently during the execution stage, CTDE methods inevitably lack a principled exploration strategy  \cite{matignon2007hysteretic}.  
Clearly, an increasing per-agent exploration rate of $\epsilon$-greedy in the single-agent setting  can help exploration; however, it has been proved \cite{mahajan2019maven} that due to the structural constraints  (e.g. the monotonicity assumption in QMIX), in the multi-agent setting, increasing $\epsilon$ will only lower the probability of obtaining the optimal value function. 
As a treatment, MAVEN \cite{mahajan2019maven} introduces a hierarchical  model to coordinate diverse explorations among agents.  
Yet, a principled exploration strategy with minor structural constraints on the value function is still missing for value-based CTDE methods. 

To eliminate restrictive constraints on the  value function decomposition, one reasonable solution  is to make agents acquire a \textbf{diverse} set of behavioral models during training so that  the optimal action of one agent  does not depend on the actions of the other agents. 
In such  scenario,  the equivalence between the local maxima on the per-agent value function  and the global maximum on the joint value function can be automatically achieved.
As a result,  the centralized value function can enjoy a natural  factorization with no need for any structural constraints beforehand.

In this paper, we present a new value-based solution in the CTDE paradigm to     multi-agent cooperative tasks. 
We establish Q-DPP, an extension of  determinantal point process (DPP) \cite{macchi1977fermion} with partition constraint, and apply it to  multi-agent learning. 
DPPs  are elegant probabilistic models on sets that can capture both quality and diversity when a subset is sampled from a ground set; this makes them ideal for modeling  the set that contains  different agents' observation-action pairs in the multi-agent learning context.  
 We adopt Q-DPP as a  function  approximator for the centralized value function. 
 An attractive  property of using Q-DPP is that, when reaching the optimum,  it can offer a natural factorization on the centralized value function, assuming agents have acquired a diverse set of behaviors.
 Our method eliminates the need for \emph{a priori}  structural constraints or bespoke neural architectures.
In fact, we demonstrate that Q-DPP generalizes current solvers  including VDN, QMIX, and QTRAN, where all these methods can be derived as special cases from Q-DPP.
As an additional contribution, we  adopt a tractable sampler, based on the idea of sample-by-projection in $P$-DPP   \cite{celis2018fair}, for Q-DPP  with theoretical approximation guarantee. 
 Our sampler makes agents explore in a sequential manner; agents who act later are coordinated to visit only the orthogonal areas in the state space that previous agents haven't explored. 
Such coordinated way of explorations effectively boosts the sampling efficiency in the CTDE setting. 
Building upon these  advantages, finally, we demonstrate that our proposed Q-DPP algorithm is superior to the existing  state-of-the-art solutions on a variety of multi-agent cooperation benchmarks.

\section{Preliminaries: Determinantal Point Process}

DPP is a probabilistic framework that characterizes how likely a subset  is going to be sampled from a ground set. 
Originated from quantum physics for modeling repulsive Fermion particles \cite{macchi1977fermion}, 
DPP has recently been introduced  to the machine learning community due to its probabilistic nature \cite{kulesza2012determinantal}. 
\begin{definition}[DPP]
\label{def:dpp}
For a ground set of items $\mathcal{Y}=\{1,2,\dots, M\}$, a DPP, denoted by $\mathbb{P}$, is  a probability measure on  the set of all subsets of $\mathcal{Y}$, i.e., $2^{\mathcal{Y}}$. 
Given  an  $M \times M$ positive semi-definite (PSD) kernel $\bm{\mathcal{L}}$ that measures similarity for any pairs of items in $\mathcal{Y}$, 
let $\bm{Y}$ be a random subset drawn according to $\mathbb{P}$, then we  have,  $\forall Y \subseteq  \mathcal{Y}, $
\begin{equation}
\mathbb{P}_{\bm{\mathcal{L}}}\big(\bm{Y} = Y\big) \propto \det	\big(\bm{\mathcal{L}}_Y\big) = \operatorname{Vol}^2\big(\{\bm{w}_i\}_{i \in Y}\big),
\label{def:dpp}
\end{equation}
where $\bm{\mathcal{L}}_Y :=  [\bm{\mathcal{L}}_{i,j}]_{i,j \in Y}$  denotes the submatrix of $\bm{\mathcal{L}}$  whose entries are indexed by the items included in $Y$. 
If we write $\bm{\mathcal{L}}:=\bm{\mathcal{W}}\bm{\mathcal{W}}^{\top}$ with $\bm{\mathcal{W}} \in \mathbb{R}^{M \times P}, P \le M$, and rows of $\bm{\mathcal{W}}$ being $\{\bm{w}_i\}$, then the determinant value is essentially the squared $|Y|$-dimensional volume of parallelepiped spanned by the rows of $\bm{\mathcal{W}}$ corresponding to elements in $Y$. 
\end{definition}
A PSD matrix ensures  all principal minors of $\bm{\mathcal{L}}$ are non-negative $\det	(\bm{\mathcal{L}}_Y) \ge 0 $; it thus suffices to be a proper probability distribution. 
The normalizer can be computed as:  $\sum_{Y \subseteq  \mathcal{Y}}\det	(\bm{\mathcal{L}}_Y) = \det	(\bm{\mathcal{L}} + \bm{I})$, where $\bm{I}$ is an  $M\times M$ identity matrix. 
Intuitively, one can think of a diagonal entry $\bm{\mathcal{L}}_{i,i}$ as capturing the quality of item $i$, while an off-diagonal entry $\bm{\mathcal{L}}_{i,j}$ measures the similarity between items $i$ and $j$.
DPP models the \textbf{repulsive} connections among \textbf{multiple} items in a sampled subset. In the example of two items, $\mathbb{P}_{\bm{\mathcal{L}}}(\{i, j\}) \propto \left|\begin{array}{ll}{\bm{\mathcal{L}}_{i,i}} & {\bm{\mathcal{L}}_{i,j}} \\ {\bm{\mathcal{L}}_{j,i}} & {\bm{\mathcal{L}}_{j,j}}\end{array}\right|=  \bm{\mathcal{L}}_{i,i}\bm{\mathcal{L}}_{j,j}-\bm{\mathcal{L}}_{i,j}\bm{\mathcal{L}}_{j,i}$, which suggests, if item $i$ and item $j$ are perfectly  similar, such that  $\bm{\mathcal{L}}_{i,j}=\sqrt{\bm{\mathcal{L}}_{i,i}\bm{\mathcal{L}}_{j,j}}$, then we know these two items will almost surely not co-occur, hence such two-item subset of $\{i, j\}$ from the ground set will never be  sampled.

DPPs are attractive in that they only require training the kernel matrix $\bm{\mathcal{L}}$, which can be learned via maximum likelihood  \cite{affandi2014learning}.
A trainable DPP  favors many supervised learning tasks where diversified outcomes are desired,  such as  image generation \cite{elfeki2019gdpp}, video summarization \cite{sharghi2018improving},  model ensemble \cite{pang2019improving}, and recommender system \cite{chen2018fast}. 
It is, however,  non-trivial to adapt DPPs to a multi-agent setting since  additional restrictions are required to put on the ground set so that valid samples can be drawn for the purpose of multi-agent training. 
This leads to  our Q-DPPs. 

\section{Multi-Agent Determinantal Q-Learning}

We offer a new value-based solution to multi-agent cooperative tasks.  
In particular, we introduce Q-DPPs as general function approximators  for the centralized value functions, similar to neural networks in deep Q-learning   \cite{mnih2015human}. 
We start from the problem formulation. 

\subsection{Problem Formulation of Multi-Agent Cooperation}

Multi-agent cooperation in a partially-observed environment is usually modeled as a Dec-POMDP \cite{oliehoek2016concise} denoted by a tuple $\mathcal{G}=<\mathcal{S}, \mathcal{N}, \mathcal{A},  \mathcal{O}, \mathcal{Z}, \mathcal{P}, \mathcal{R}, \gamma>$.
Within $\mathcal{G}$, $s \in \mathcal{S}$ denotes the global environmental state.
At every time-step  $t \in \mathbb{Z}^+$, each agent $i \in  \mathcal{N} =\{1, \ldots, N\}$ selects an action $a_i \in \mathcal{A}$ where a joint action stands for $\bm{a}:= (a_i)_{i\in\mathcal{N}} \in \mathcal{A}^N$. 
Since the environment is partially observed, each agent only has access  to its local observation  $o \in \mathcal{O}$ that is acquired through an observation function  $\mathcal{Z}(s, a): \mathcal{S} \times \mathcal{A} \rightarrow \mathcal{O}$.
The state transition dynamics are  determined by   $\mathcal{P}(s' | s, \bm{a}):= \mathcal{S} \times \mathcal{A}^N \times \mathcal{S} \rightarrow [0, 1]$.
Agents optimize towards one shared goal whose  performance is measured by $\mathcal{R}(s, \bm{a}): \mathcal{S} \times \mathcal{A}^N \rightarrow \mathbb{R}$,  and $\gamma \in [0,1)$ discounts the future rewards.
Each agent recalls an observation-action history $\tau_i \in \mathcal{T} := (\mathcal{O} \times \mathcal{A})^{t} $, and executes a stochastic policy $\pi_i(a_i | \tau_i): \mathcal{T} \times \mathcal{A} \rightarrow [0,1]$ which is conditioned on $\tau_i$. 
All of the agents histories is defined as   $\boldsymbol{\tau}:=(\tau_i)_{i\in \mathcal{N}}\in\mathcal{T}^N$.
Given a joint policy $\bm{\pi}:=(\pi_i)_{i \in \mathcal{N}}$, 
the joint action-value function at time $t$ stands as $Q^{\bm{\pi}}(\bm{\tau}^t, \bm{a}^t) = \mathbb{E}_{s^{t+1:\infty}, \bm{a}^{t+1:\infty}} [G^t | \bm{\tau}^t, \bm{a}^t]$, where $G^t = \sum_{i=0}^{\infty}\gamma^i\mathcal{R}^{t+i}$ is the total accumulative rewards. 

The \textbf{goal}  is to find an optimal value function  $Q^* = \max_{\bm{\pi}} Q^{\bm{\pi}}(\bm{\tau}^t, \bm{a}^t)$ and the corresponding policy $\bm{\pi}^*$.
A  direct centralized approach is to learn the joint value function, parameterized by $\theta$, by 
minimizing the squared temporal-difference error $\mathcal{L}(\theta)$  \cite{watkins1992q}  from a sampled mini-batch of transition data $\{\langle \bm{\tau}, \bm{a}, \mathcal{R},  \bm{\tau}' \rangle\}_{j=1}^{E}$, i.e., 
\vspace{-15pt}
\begin{equation}
\label{eq:q_loss}
\small{
\mathcal{L}(\theta)= \sum_{j=1}^{E}\Big\| \mathcal{R} + \gamma\max_{\bm{a}'} Q(\bm{\tau}', \bm{a}'; \theta^{-}) - Q^{\bm{\pi}}(\bm{\tau}, \bs{a} ; \theta) \Big\|^2,    \hspace{-5pt}
}
\end{equation}
where $\theta^{-}$ denotes the target parameters that can be  periodically copied from $\theta$ during training.

In our work, apart from  the joint value function, we also focus on obtaining a decentralized policy for each agent. 
CTDE is a paradigm for solving Dec-POMDP \cite{oliehoek2008optimal} where it allows the  algorithm  access to all of the agents local histories $\bm{\tau}$  during training. During testing, however, the algorithm uses each of the agent's own history $\tau_i$ for execution. 
CTDE methods provide valid solutions  to  multi-agent  cooperative tasks that are  \emph{decentralizable}, which is formally defined as below.  
\begin{definition}[Decentralizable Cooperative Tasks, a.k.a. IGM Condition ~\cite{son2019qtran}] 
A cooperative task is decentralizable if $\exists \{Q_i\}_{i=1}^{N}$ such that  $\forall \bm{\tau} \in \mathcal{\tau}^N, \bm{a}\in\mathcal{A}^N$,  
\label{def:igm}
\small{
\begin{equation}
\label{eq:igm}
{
\hspace{-5pt} \arg \max _{\bs{a}} Q^{\bm{\pi}}(\bm{\tau}, \bs{a})=\left[\begin{array}{c}{\arg \max _{a_{1}} Q_{1}(\tau_1, a_{1})} \\ {\vdots} \\ {\arg \max _{a_{N}} Q_{N}(\tau_N, a_{N})}\end{array} \right].\hspace{-5pt}}
\end{equation}}
\end{definition}
 Eq. \ref{eq:igm}   suggests that 
local maxima on the extracted value function per every agent needs  to amount to the global maximum on the joint value function. 
A key challenge for CTDE methods is, then, how to correctly extract each of the agent's individual Q-function $\{Q_i\}_{i=1}^{N}$, and as such an executable policy,   from  a centralized Q-function  $Q^{\bm{\pi}}$. 

To  satisfy Eq. \ref{eq:igm}, current  solutions rely on restrictive assumptions that enforce  structural constraints on the factorization of the joint Q-function. For example, 
VDN \cite{sunehag2017value} adopts the additivity assumption by assuming   $Q^{\bm{\pi}}(\bm{\tau}, \bm{a}):=\sum_{i=1}^{N}Q_i(\tau_i, a_i)$.   QMIX \cite{rashid2018qmix} applies the monotonicity assumption to ensure   $\frac{\partial Q^{\bm{\pi}}(\bm{\tau}, \bm{a})}{\partial Q_i(\tau_i, a_i)} \ge 0, \forall i \in \mathcal{N}$.
QTRAN \cite{son2019qtran}  introduces a refined factorizable learning   objective in addition to QMIX. 
Nonetheless, structural constraints harm the representational power of the centralized value function, and also hinder efficient explorations \cite{son2019qtran}.   
To mitigate these problems, we propose Q-DPP as an alternative that naturally factorizes the joint Q-function by learning a diverse set of behavioral models among agents.

\begin{figure*}[t!]
    \centering
                    \vspace{-8pt}
    \includegraphics[width=0.97\textwidth]{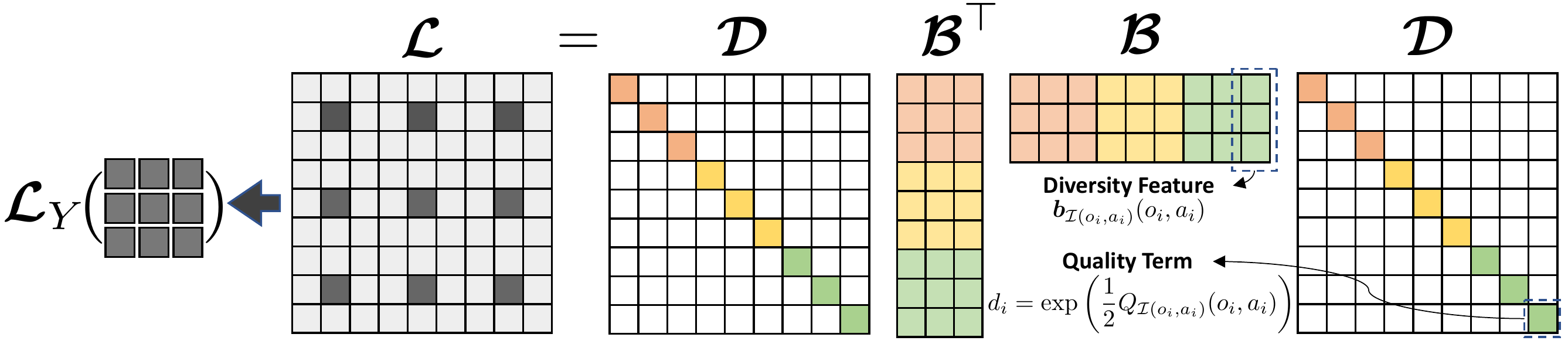}
            \vspace{-2pt}
    \caption{Example of Q-DPP with  quality-diversity kernel decomposition in a single-state three-player learning task, each agent has three actions to choose. 
    The size of the ground set is $|\mathcal{Y}|=9$, and the size of valid subsets $|\mathcal{C}(\bm{o})|$ is  $3^3=27$. 
    Different colors represent different partitions of each agent's observation-action pairs. 
    Suppose all three agents select the $2$nd action, then the Q-value of the joint action according to Eq.~\ref{eq:q-dpp} is    $Q^{\bm{\pi}}\big(\bm{o}, \bs{a}\big) = 
    \det\big([\bm{\mathcal{L}}_{[i,j], i,j \in\{2,5,8\} }]\big).$ 
    }
    \label{fig:dpp}
    \vspace{-10pt}
\end{figure*}

\subsection{Q-DPP: A Constrained DPP for MARL} 
Our method is established on Q-DPP which is  an extension of DPP that suits MARL.  
We assume that local observation $o_i$ encodes all  history information  $\tau_i$ at each  time-step.
We model the ground set of all agents' observation-action pairs by a DPP, i.e.,   $\mathcal{Y}=\big\{(o_1^1, a_1^1), \ldots,(o_N^{|\mathcal{O}|}, a_N^{|\mathcal{A}|}) \big\}$ with the size of the ground set being  $|\mathcal{Y}|=N|\mathcal{O}||\mathcal{A}|$.

In the context of multi-agent learning,  each agent takes one valid action depending on its local observation. 
A valid sample from DPP, therefore, is expected to include one valid observation-action pair  for each agent, 
and the observations from the sampled pairs must match the true observations  that agents receive at every time step.  
To meet such requirements, we  propose a new   type of DPP, named  \textbf{Q-DPP}. 
\begin{definition}[Q-DPP]
\label{def:qdpp}
Given a ground set $\mathcal{Y}$ of size $M$ that includes $N$ agents' all possible observation-action pairs $\mathcal{Y}=\big\{(o_1^1, a_1^1), \ldots, (o_N^{|{\mathcal{O}|}}, a_N^{|\mathcal{A}|})\big\}$, we  partition $\mathcal{Y}$ into $N$ disjoint parts, i.e.,   $\mathcal{Y} = \bigcup_{i=1}^{N} \mathcal{Y}_i$ and $\sum_{i=1}^N |\mathcal{Y}_i|= M = N |\mathcal{O}||\mathcal{A}|$, where each partition represents each individual agent's all possible observation-action pairs. 
At every time-step, given agents' observations,     $\bm{o}=(o_i)_{i \in \mathcal{N}}$, 
we define $\mathcal{C}(\bm{o}) \subseteq \mathcal{Y}$ to be a set of valid subsets  including only  observation-action pairs that  agents are allowed to take,    
\[\mathcal{C}(\bm{o}):=\big\{Y \subseteq  \mathcal{Y}: |Y \cap \mathcal{Y}_i(o_i) |=1, \forall i \in \{1,\ldots,N\} \big\}, \] 
with $|\mathcal{C}(\bm{o})|=|\mathcal{A}|^N$, and 
$\mathcal{Y}_i(o_i)$ of size $|\mathcal{A}|$ denotes  the set of  pairs in partition $\mathcal{Y}_i$ with only $o_i$ as the observation,  $$\mathcal{Y}_i(o_i)=\big\{(o_i, a_i^1),\ldots,(o_i, a_i^{|\mathcal{A}|})\big\}. $$
Q-DPP, denoted by $\tilde{\mathbb{P}}$, defines a  probability measure over the valid subsets $Y \in \mathcal{C}(\bm{o}) \subseteq \mathcal{Y}$. 
Let $\bm{Y}$ be a random subset drawn according to $\tilde{\mathbb{P}}$, its probability distribution is defined: 
\begin{equation}
	\label{def:q-dpp}
	\tilde{\mathbb{P}}_{\bm{\mathcal{L}}}\big(\boldsymbol{Y}=Y |   \boldsymbol{Y} \in \mathcal{C}(\bm{o})\big) := \dfrac{\operatorname{det}\big(\bm{\mathcal{L}}_{Y}\big)}{\sum_{_{Y' \in \mathcal{C}(\bm{o})}}\operatorname{det}\big(\bm{\mathcal{L}}_{Y'}\big)}\ . 
	\end{equation}
In addition, given a valid sample $Y \in \mathcal{C}(\bm{o})$,  
 we define an identifying  function $\mathcal{I}: Y \rightarrow \mathcal{N}$ that specifies the  agent number for each valid pair in $Y$, and an index function $\mathcal{J}: \mathcal{Y} \rightarrow \{1,\ldots,M\}$ that specifies the cardinality of each item in $Y$  in the ground set $\mathcal{Y}$. 
 \vspace{-5pt}
\end{definition}

The construction of Q-DPP is inspired by 
$P$-DPP \cite{celis2018fair}.
However,  the partitioned sets in $P$-DPP stay fixed, while in Q-DPP, $\mathcal{C}(\bs{o_t})$ changes at every time-step with the new observation, and the kernel $\bm{\mathcal{L}}$ is learned through the process of  reinforcement learning rather than being given.  
More differences are listed  in \emph{Appendix \ref{sec:kdpp_qdpp}}. 

Given Q-DPPs, we can represent the centralized value function by adopting Q-DPPs as general function approximators:    
\begin{equation}
Q^{\bm{\pi}}\big(\bm{o}, \bs{a}\big) :=  \log \det \Big( \bm{\mathcal{L}}_{Y=\big\{(o_1, a_1), \ldots, (o_N, a_N)  \big\}  \in \mathcal{C}(\bm{o}^t)}\Big), \hspace{-5pt}
\label{eq:q-dpp}
\end{equation}
where $\bm{\mathcal{L}}_Y$ denotes the sub-matrix of $\bm{\mathcal{L}}$  whose entries are indexed by the pairs  included in $Y$. 
Q-DPP embeds the connection between the joint action and each agent's individual actions  into a subset-sampling process, and the Q-value is quantified by the determinant of a kernel matrix whose elements are indexed by the associated  observation-action pairs. 
The goal of multi-agent learning is to learn an optimal joint  Q-function. 
Eq.~\ref{eq:q-dpp} states $\det (\bm{\mathcal{L}}_Y) = \exp(Q^{\pi}(\bm{o}, \bm{a}))$,  meaning Q-DPP actually assigns large probability to the subsets that have large Q-values. 
Given $\det	(\bm{\mathcal{L}}_Y) $ is always positive,  the $\log$ operator ensures Q-DPPs, as general function approximators, can recover any real  Q-functions.

DPPs can capture both the quality and diversity of a sampled subset; 
the joint Q-function represented by Q-DPP in theory should not only acknowledge the quality of each agent's individual action towards a large reward, but the diversification of agents' actions as well. 
 The remaining question is, then,  how to obtain such quality-diversity representation. 

\subsection{Representation of Q-DPP Kernels} 
For any PSD matrix $\bm{\mathcal{L}}$, such a $\bm{\mathcal{W}}$ can always be found so that  $\bm{\mathcal{L}}=\bm{\mathcal{W}}  \bm{\mathcal{W}}^{\top}$ where $\bm{\mathcal{W}} \in \mathbb{R}^{M \times P}, P \le M$. 
Since the diagonal and off-diagonal entries of $\bm{\mathcal{L}}$ represent \emph{quality} and \emph{diversity} respectively, we adopt an  interpretable  decomposition  by expressing each row  of  $\bm{\mathcal{W}}$ as   a  product of a \textbf{quality} term $d_{i} \in \mathbb{R}^{+}$ and a  \textbf{diversity} feature term $\bs{b}_{i} \in \mathbb{R}^{P\times 1}$ with $\left\|\bs{b}_{i}\right\|\le1$, i.e., $\bs{w}_{i} = d_{i}\bs{b}_{i}^{\top}$. 
An example of such decomposition is visualized in Fig.~\ref{fig:dpp} where we define $\bm{\mathcal{B}}=[\bm{b}_1,\ldots,\bm{b}_M]$ and $\bm{\mathcal{D}} = \diag(d_1,\ldots,d_M).$ 
Note that   both $\bm{\mathcal{D}}$ and $\bm{\mathcal{B}}$ are free parameters that can be learned from the environment during the Q-learning process in Eq.~\ref{eq:q_loss}.

If we denote the quality term as each agent's individual Q-value for a given observation-action pair, i.e.,  $\forall (o_i, a_i)  \in \mathcal{Y}, i=\{1,\ldots,M\}, d_i := \exp\big(\frac{1}{2} Q_{\mathcal{I}(o_i, a_i)}(o_i, a_i)\big)$, then Eq.~\ref{eq:q-dpp} can be further written  into 
\begin{equation}
\small{
\begin{aligned}[b]
     Q^{\bm{\pi}}\big(\bm{o}, \bs{a}\big) 
    = & \log \det \Big(\bm{\mathcal{W}}_Y \bm{\mathcal{W}}_Y^{\top} \Big) \\
    = &  \log  \Big(\tr \big(\bm{\mathcal{D}}_Y^{\top}\bm{\mathcal{D}}_Y  \big) \operatorname{det}\big(\bm{\mathcal{B}}_{Y}^{\top}\bm{\mathcal{B}}_Y\big)\Big)  \\[-2pt]
   = & \sum_{i=1}^{N}Q_{\mathcal{I}(o_i, a_i)} \big(o_i, a_i\big) + \log  \det \big(\bm{\mathcal{B}}_{Y}^{\top}\bm{\mathcal{B}}_Y\big). 
\label{eq:qd_decomp}
\end{aligned}
}
\vspace{-5pt}
\end{equation}
 Since a determinant value only reaches the maximum when the associated  vectors in $\bm{\mathcal{B}}_Y$ are mutually orthogonal \cite{noble1988applied},   
Eq.~\ref{eq:qd_decomp} essentially stipulates that Q-DPP represents the joint value function by taking into account not only the quality of each agent's contribution towards reward maximization, more importantly, from a holistic perspective, the  orthogonalization of agents' actions.


In fact, the inclusion of diversifying agents'  behaviors is an important factor in satisfying the condition in   Eq.~\ref{eq:igm}. 
Intuitively, in a decentralizable task with a shared goal, promoting orthogonality between agent's actions can help  clarify the functionality and  responsibility of each agent, which in return leads to a better instantiation  of Eq.~\ref{eq:igm}.
On the other hand, diversity does not  means that agents have to take different actions all the time. 
 Since the goal is still to achieve large reward via optimizing Eq.~\ref{eq:q_loss}, certain scenarios, such as agents need to take identical actions to accomplish a task, will not be excluded as a result of promoting diversity.

 \subsection{Connections to Current Methods}

Based on the quality-diversity representation, 
one can draw  a key connection between Q-DPP and the existing  methods. 
 It turns out that, under the sufficient condition that if the learned diversity features that correspond to the optimal actions are mutually orthogonal, then Q-DPP 
 degenerates to VDN \cite{sunehag2017value}, QMIX \cite{rashid2018qmix},  and QTRAN \cite{son2019qtran} respectively.
 
 To elaborate such condition, 
 let us  denote $a_i^* = \arg\max Q_i(o_i, a_i)$, $\bm{a}^* = (a_i^*)_{i\in\mathcal{N}}$, $Y^* = \{(o_i, a_i^*)\}_{i=1}^{N}$, with  $\|\bs{b}_{i}\|=1$ and $\bs{b}_i^{\top}\bs{b}_j =0, \forall i\neq j$, then we have 
 \begin{equation}
 \label{eq:bbstar}
 \det \big(\bm{\mathcal{B}}_{Y^*}^{\top}\bm{\mathcal{B}}_{Y^*}\big) = 1. 
 \end{equation}
\textbf{Connection to VDN}. 
When $\{ \bm{b}_j \}_{j=1}^{M}$ are pairwise orthognal, by plugging Eq.~\ref{eq:bbstar} into Eq.~\ref{eq:qd_decomp}, we can obtain   
{\small
\begin{equation}
	Q^{\bm{\pi}}\big(\bm{o}, \bs{a}^*\big)= \sum_{i=1}^N Q_{\mathcal{I}(o_i, a_i^*)}\big(o_i, a_i^*\big).
	\label{eq:vdn}
\end{equation}}
Eq.~\ref{eq:vdn} recovers the exact additivity constraint that VDN applies to factorize the joint value function in meeting Eq.~\ref{eq:igm}.

\textbf{Connection to QMIX}. 
Q-DPP also generalizes QMIX, which adopts a monotonic constraint on the centralized value function to meet Eq.~\ref{eq:igm}. 
Under the special condition when $\{ \bm{b}_j \}_{j=1}^{M}$ are mutually orthogonal, we can easily show that Q-DPP meets the monotonicity condition because  
{\small 
\begin{equation}
\label{eq:qmix}
	\dfrac{\partial Q^{\bm{\pi}}\big(\bm{o}, \bs{a}^*\big)}{\partial Q_{\mathcal{I}(o_i, a_i^*)}\big(o_i, a_i^*\big)} = 1 \ge 0, \ \ \ \forall \mathcal{I}(o_i, a_i^*) \in \mathcal{N}. 
\end{equation}}

\textbf{Connection to QTRAN}. 
Q-DPP also meets the sufficient conditions that QTRAN proposes for meeting Eq.~\ref{eq:igm}, that is, 
{\small
\begin{equation}
\sum_{i=1}^{N} Q_{i}\big(o_{i}, a_{i}\big)-Q^{\bm{\pi}}(\boldsymbol{o}, \boldsymbol{a})+V(\boldsymbol{o})=\Big\{\begin{array}{ll}{0} & {\boldsymbol{a}=\boldsymbol{a}^*} \\ {\geq 0} & {\boldsymbol{a} \neq \boldsymbol{a}^*}\end{array}, 
\label{eq:qtran}
\end{equation}}
where $V(\bm{o})=\max _{\boldsymbol{a}} Q^{\bm{\pi}}(\boldsymbol{o}, \boldsymbol{a})-\sum_{i=1}^{N} Q_{i}\big(o_i, a^*_{i}\big)$. 
Through Eq. \ref{eq:qd_decomp}, we know Q-DPP can have   Eq.~\ref{eq:qtran} written as 
\begin{equation}
\label{eq:qdppqtran}
	\hspace{-5pt} - \log  \det \big(\bm{\mathcal{B}}_{Y}^{\top}\bm{\mathcal{B}}_Y\big) + \max _{\boldsymbol{a}} Q^{\bm{\pi}}(\boldsymbol{o}, \boldsymbol{a})-\sum_{i=1}^{N} Q_{i}\big(o_{i}, a^*_{i}\big).  \hspace{-5pt}
\end{equation}When $\bm{a}=\bm{a}^*$, for pairwise orthogonal $\{ \bm{b}_j \}_{j=1}^{M}$, Q-DPP satisfies the first condition since Eq.~\ref{eq:qdppqtran} equals to zero due to  $\log \det (\bm{\mathcal{B}}_{Y^*}^{\top}\bm{\mathcal{B}}_{Y^*}) = 0$. 
 When $\bm{a} \neq \bm{a}^*$, Eq.~\ref{eq:qdppqtran} equals to $- \log  \det \big(\bm{\mathcal{B}}_{Y}^{\top}\bm{\mathcal{B}}_Y\big) + \log  \det \big(\bm{\mathcal{B}}_{Y^*}^{\top}\bm{\mathcal{B}}_Y^*\big)$, which is always positive  since $\det (\bm{\mathcal{B}}_{Y}^{\top}\bm{\mathcal{B}}_{Y}) < 1, \forall Y \neq Y^*$; Q-DPP thereby  meets the second condition of Eq.~\ref{eq:qtran} and recovers QTRAN.  
 
\textbf{Other Related Work}.
Determinantal SARSA \cite{osogami2019determinantal} applies a normal DPP to model the ground set of the joint state-action pairs $\big\{(s^0, a_1^0, \ldots, a_N^0), \ldots,  (s^{|\mathcal{S}|}, a_1^{|\mathcal{A}|}, \ldots, a_N^{|\mathcal{A}|})\big\}$. It fails to consider at all a proper  ground set that suits multi-agent problems, which leads to the size of  subsets being $2^{|\mathcal{S}||\mathcal{A}|^N}$ that is double-exponential to the number of agents.
Furthermore, unlike Q-DPP that learns decentralized policies, Det. SARSA  learns the centralized  joint-action policy, which strongly limits its applicability for  scalable real-world  tasks.

\begin{algorithm}[h]
\begin{algorithmic}[1]
\STATE \underline{\textbf{DEF} Orthogonalizing-Sampler ($\mathcal{Y}, \bm{\mathcal{D}}, \bm{\mathcal{B}}, \bm{o}$):~~~~~~~~~~~~~~~}\;
	\STATE \textbf{Init}: $ \bm{b}_j \leftarrow \bm{\mathcal{B}}_{[:, j]}, Y \leftarrow \emptyset, B \leftarrow \emptyset, J \leftarrow \emptyset .$\;
    \FOR{each partition $\mathcal{Y}_i$}
        \STATE  Define $\ \ \forall (o, a) \in  \mathcal{Y}_i(o_i)  $\\$q(o, a):=\big\|\bs{b}_{\mathcal{J}(o, a)}\big\|^{2}\exp{\big(\bm{\mathcal{D}}_{\mathcal{J}(o, a), \mathcal{J}(o, a)} \big)}$.  \; \vspace{2pt}
        \STATE  Sample $(\tilde{o}_i, \tilde{a}_i) \in \mathcal{Y}_i(o_i)$ from the distribution: $$\bigg\{\dfrac{q(o, a)}{\sum_{{\color{red}(\hat{o}, \hat{a}) \in \mathcal{Y}_i(o_i)} } q(\hat{o},\hat{a})}\bigg\}_{{(o, a) \in \mathcal{Y}_i(o_i)}}.$$\; \vspace{-6pt}
        \STATE  Let $Y \leftarrow Y \cup (\tilde{o}_i, \tilde{a}_i), B \leftarrow B \cup \bs{b}_{\mathcal{J}(\tilde{o}_i, {\tilde{a}_i})},$ \\ $\ \ \ \ \ \ J \leftarrow J \cup \mathcal{J}(\tilde{o}_i, {\tilde{a}_i})$.\;
        \STATE // \emph{Gram-Schmidt orthogonalization} 
        \STATE Set $\bs{b}_j=\amalg_{\operatorname{span}\{B\}}\left(\bs{b}_j\right), \forall j \in \{1,...,M\} - J$\ \     \; 
    \ENDFOR 
    \STATE \textbf{Return}: $Y$.\; \\
    \vspace{-8pt}
    \underline{~~~~~~~~~~~~~~~~~~~~~~~~~~~~~~~~~~~~~~~~~~~~~~~~~~~~~~~~~~~~~~~~~~~~~~~~~~~~~~~~~~~~~~}
    \vspace{-5pt}
    \STATE 
\underline{~~~~~~~~~~~~~~~~~~~~~~~~~~~~~~~~~~~~~~~~~~~~~~~~~~~~~~~~~~~~~~~~~~~~~~~~~~~~~~~~~~~~~~}
\STATE \underline{\textbf{DEF} Determinantal-Q-Learning $(\theta=[\theta_{\bm{\mathcal{D}}}, \theta_{\bm{\mathcal{B}}}], \mathcal{Y} $):~~~~}\;
\STATE \textbf{Init}: $\theta^{-}\leftarrow\theta$, $D \leftarrow \emptyset$.
\FOR{each time-step}
    \STATE Collect observations $\bm{o} = [o_1, \ldots,o_N]$ for all agents.\;
    \STATE $\bs{a}=\textbf{\text{Orthogonalizing-Sampler}}(\mathcal{Y}, \theta_{\bm{\mathcal{D}}}, \theta_{\bm{\mathcal{B}}}, \bm{o})$.\;
    \STATE Execute $\bs{a}$, store the transition $\langle \bm{o}, \bm{a}, \mathcal{R}, \bm{o}'\rangle$ in $D$.\;
    \STATE Sample a  mini-batch of $\{\langle \bm{o}, \bm{a}, \mathcal{R},  \bm{o}' \rangle\}_{j=1}^{E}$ from $D$.\;
    \STATE Compute for each transition in the mini-batch $\ \ \ \ \ \ \ $   \\  $ \max_{\bm{a}{'}} Q\big(\bm{o}', \bm{a}{'}; \theta^{-}\big)$\\ $\ \ \ \ \ =  \log \det \big( \bm{\mathcal{L}}_{Y=\{(o_1', a_1^*),..., (o_N', a_N^*)\}}\big)$ \\\vspace{3pt} where $\ \ \ \ $  // \emph{off-policy decentralized execution}    $\ \ \ \ \ \ \ \ $\\ \vspace{1pt} { $a_i^* = \arg \max_{a_i \in \mathcal{A}_i}\Big[\big\|\theta^{-}_{\bs{b}_{\mathcal{J}(o_i', a_i)}}\big\|^{2}$ \\ $\hspace{85pt} \cdot \exp{\big(\theta^{-}_{\bm{{\mathcal{D}}}_{\mathcal{J}(o_i', a_i), \mathcal{J}(o_i', a_i)}}}\big)\Big]. $}\;    
    \STATE // \emph{centralized training}
    \STATE Update $\theta$ by minimizing $\mathcal{L}(\theta)$ defined in Eq.~\ref{eq:q_loss}.\;
    \STATE Update target $\theta^{-}=\theta$  periodically.
\ENDFOR
\STATE \textbf{Return}: $\theta_{\bm{\mathcal{D}}}, \theta_{\bm{\mathcal{B}}}$.\; \\
\end{algorithmic}
\caption{Multi-Agent Determinantal Q-Learning}
\label{algo:main_algo1}
\end{algorithm}

\subsection{Sampling from Q-DPP}
\label{sec:sample}
Agents need to explore the environment effectively during  training; however,  how to  sample from  Q-DPPs defined in Eq.~\ref{def:q-dpp} is still unknown. 
In fact,  sampling from the DPPs with partition-matroid constraint is a non-trivial task.
So far, the best known exact sampler for partitioned DPPs has $\mathcal{O}(m^p)$ time complexity with $m$ being the ground-set size and $p$ being the number of partitions \cite{li2016fast,celis2017complexity}. 
Nonetheless, these samplers still pose great computational challenges for multi-agent learning tasks and cannot scale to large number of agents because we have  $m=|\mathcal{C}(\bm{o})|=|\mathcal{A}|^N$ for multi-agent learning tasks. 

In this work,  we instead adopt a biased yet tractable sampler for Q-DPP.
Our sampler is an application of the sampling-by-projection idea in \citet{celis2018fair} and \citet{chen2018fast} which  leverages the property that Gram-Schmidt process preserves the determinant.
One benefit of our sampler is that it  promotes  efficient explorations among agents during training.
Importantly, it enjoys only linear-time complexity \emph{w.r.t.} the number of agents.
The intuition is as follows. 

\textbf{Additional Notations}. 
In a Euclidean space $\mathbb{R}^n$ equipped with an inner product $\langle\bm{\cdot}, \bm{\cdot}\rangle$, 
let  $\mathcal{U} \subseteq \mathbb{R}^n$ be any linear subspace, and $\mathcal{U}^{\perp}$ be its orthogonal complement  $\mathcal{U}^{\perp}:= \{x \in \mathbb{R}^n | \langle x, y\rangle=0,  \forall y \in \mathcal{U}\}$.
We define an orthogonal projection operator,  $\amalg_{\mathcal{U}}: \mathbb{R}^n \rightarrow \mathbb{R}^n$, such that $\forall \bm{u} \in \mathbb{R}^n, \text{if } \bm{u}=\bm{u}_1+\bm{u}_2 \text{ with } \bm{u}_1 \in \mathcal{U} \text{ and } \bm{u}_2 \in \mathcal{U}^{\perp}, \text{ then } \amalg_{\mathcal{U}}(\bm{u})=\bm{u}_2$.

Gram-Schmidt  \cite{noble1988applied} is a process for orthogonalizing a set of vectors; given a set of linearly independent vectors $\{\bm{w}_i\}$, it outputs a mutually orthogonal set of vectors $\{\hat{\bm{w}}_i \}$ by computing 
$\hat{\bm{w}}_i: =\amalg_{\mathcal{U}_i} (\bm{\bm{w}}_i)$ where  $\mathcal{U}_i = \operatorname{span}\{\bm{w}_1,\ldots,\bm{w}_{i-1}\}$.
Note that we neglect the normalizing step of Gram-Schmidt in this work.
Finally, if the rows $\{\bm{w}_i\}$ of a matrix $\bm{\mathcal{W}}$ are mutually orthogonal,  we can compute the determinant by $\det(\bm{\mathcal{W}}\bm{\mathcal{W}}^{\top})= \prod \|\bm{w}_i\|^2$.
The Q-DPP sampler is built upon the following property. 
\begin{proposition}[Volume preservation of Gram-Schmidt, see Chapter 7 in \citet{shafarevich2012linear}, also Lemma 3.1 in \citet{celis2018fair}.]
Let $\mathcal{U}_i = \operatorname{span}\{\bm{w}_1,\ldots,\bm{w}_{i-1}\}$ and $\bm{w}_i \in \mathbb{R}^P$ be the $i$-th  row of   $\bm{\mathcal{W}} \in \mathbb{R}^{M\times P}$,  then $  \prod_{i=1}^{M}\|\amalg_{\mathcal{U}_i}(\bm{w}_i)\|^2 = \det(\bm{\mathcal{W}}\bm{\mathcal{W}}^{\top})$. 
\label{prop1:gramschmidt}
\vspace{-5pt}
\end{proposition}
We also provide an intuition by Gaussian elimination in \emph{Appendix \ref{app:prop1}}. 
Proposition \ref{prop1:gramschmidt}  suggests that the determinant  of a Gram matrix is invariant to applying the Gram-Schmidt orthogonalization on the rows of that Gram matrix. 
In Q-DPP's case, a kernel matrix with  mutually orthogonal rows can largely simplify the sampling process.  
In such scenario, an effective sampler can be that, from each partition $\mathcal{Y}_i$, sample an item $i \in \mathcal{Y}_i$ with $\mathbb{P}(i) \propto \|d_i\bm{b}_i^{\top}\|^2$, then add $i$ to the output sample $Y$ and move to the next partition; the above steps iterate until all partitions are covered.
It is effortless to see that the probability of obtaining  sample  $Y$ in such a way  is 
\begin{align}
	\mathbb{P}(Y)& \propto \prod_{i \in Y} \|d_i\bm{b}_i^{\top}\|^2=\prod_{i \in Y} \|\bm{w}_i\|^2 = \det(\bm{\mathcal{W}}_{Y}\bm{\mathcal{W}}_{Y}^{\top}) \nonumber \\
	& \propto \det(\bm{\mathcal{L}}_{Y}). \hspace{-15pt}
\end{align}
We formally describe the  orthogonalizing sampling procedures in Algorithm \ref{algo:main_algo1}. 
As it is suggested in \citet{celis2018fair}, 
the time complexity of the sampling function is $\mathcal{O}(NMP)$ (see also the breakdown analysis for each step in \emph{Appendix} \ref{timecomp}), given the input size being $\mathcal{O}(MP)$, our sampler thus enjoys linear-time complexity \emph{w.r.t} the agent number.

Though the Gram-Schmidt process can preserve the determinant and simply the sampling process, it comes at a prize of introducing \textbf{bias} on the normalization term. 
Specifically, the normalization in our proposed sampler is conducted at each agent/partition level $\mathcal{Y}_i(o_i)$ (see the red in line $5$) which does not match Eq.~\ref{def:q-dpp} that  suggests normalizing by listing  all valid samples considering all partitions $\mathcal{C}({\bm{o}})$; this directly leads to a sampled subset from our sampler having \emph{larger} probability  than what Q-DPP defines.
Interestingly, it turns out that such violation  can be controlled through bounding the singular values of each partition in the kernel matrix (see Assumption \ref{assump:single}), a technique  also known as the $\beta$-balance condition introduced in $P$-DPP  \cite{celis2018fair}.
\begin{assumption}[Singular-Value Constraint on Partitions]
\label{assump:single}
	For a Q-DPP defined in Definition \ref{def:dpp}, which is parameterized by $\bm{\mathcal{D}} \in \mathbb{R}^{M \times M}, \bm{\mathcal{B}} \in \mathbb{R}^{P \times M}$ and $\bm{\mathcal{W}}: = \bm{\mathcal{D}}\bm{\mathcal{B}}^{\top}$, let $\sigma_1 \ge \ldots \ge \sigma_P$ represent the singular values of $\bm{\mathcal{W}}$, and $\hat{\sigma}_{i, 1} \ge \ldots \ge \hat{\sigma}_{i, P}$ denote the singular values of $\bm{\mathcal{W}}_{\mathcal{Y}_i}$ that is the submatrix of $\bm{\mathcal{W}}$ with the rows and columns corresponding  to the $i$-th partition $\mathcal{Y}_i$\ , we assume $\forall j\in \{1,\ldots,P\},  \ \exists \  \delta \in (0, 1] \ \text{, s.t., \ }  \min_{i\in \{1,\ldots,N\}} \hat{\sigma}_{i, j}^2 / \delta \ge \sigma_j^2$  holds.   
\end{assumption}
\begin{theorem}[Approximation Guarantee of Orthogonalizing Sampler]
For a Q-DPP defined in Definition \ref{def:dpp}, under Assumption \ref{assump:single}, the Orthogonalizing Sampler described in Algorithm \ref{algo:main_algo1} returns a sampled subset $Y \in \mathcal{C}(\bm{o})$ with probability $\mathbb{P}(Y) \le  1/\delta^N \cdot \tilde{\mathbb{P}}(\bm{Y}=Y)$ where $N$ is the number of agents, $\tilde{\mathbb{P}}$ is defined in Eq. \ref{def:q-dpp}, $\delta$ is defined in Assumption \ref{assump:single}. 
\label{theorem:sample}
\end{theorem}
\begin{proof}
\vspace{-8pt}
The  proof is  in \emph{Appendix A.2}. It can also be taken as a special case of Theorem 3.2 in \citet{celis2018fair} when the number of sample from each partition is one. 
\vspace{-8pt}
\end{proof}
Theorem \ref{theorem:sample} effectively suggests a  way to bound the error between our sampler and the true distribution of Q-DPP through  minimizing the difference between $\sigma_j^2$ and $\hat{\sigma}_{i, j}^2$. 

\subsection{Determinantal Q-Learning}

We present the full learning procedures in Algorithm \ref{algo:main_algo1}. 
Determinantal Q-Learning is a CTDE method. 
During training,  agents' explorations are conducted through the orthogonalizing-sampler. 
 The parameters of $\bm{\mathcal{B}}$ and $\bm{\mathcal{D}}$ are updated through Q-learning  in a centralized way by following Eq.~\ref{eq:q_loss}.  
 To meet Assumption \ref{assump:single},  one can implement   an auxiliary loss function of $\max (0, \sigma_j^2 - \hat{\sigma}_{i, j}^2/\delta)$  in addition to   where $\delta$ is a hyper-parameter.
 Given Theorem \ref{theorem:sample}, for large $N$, we know $\delta$ should be set  close to $1$ to make the bound tight.
 In fact, it is worth mentioning that the Gram-Schmidt process adopted in the sampler can  boost the sampling efficiency for multi-agent  training.  
Since agents' diversity features of observation-action pairs are orthogonalized every time after a partition is visited,  agents who act later are essentially coordinated to explore the observation-action space that is distinctive to all previous agents. This speeds up training in early stages. 

During execution, agents only need to access  the parameters in their  own partitions  to compute the greedy action (see line $19$).   
Note that neural networks can be seamlessly applied to represent both $\bm{\mathcal{B}}$ and $\bm{\mathcal{D}}$ to tackle continuous states. Though a full treatment of deep Q-DPP needs substantial future work, we show a proof of concept in \emph{Appendix \ref{deep-qdpp}}. 
Hereafter, we use Q-DPP to represent our  proposed  algorithm.

\section{Experiments}

We compare Q-DPP with  state-of-the-art CTDE solvers for multi-agent cooperative tasks, including COMA~\citep{COMA}, VDN~\citep{sunehag2017value},  QMIX~\citep{rashid2018qmix}, QTRAN  ~\citep{son2019qtran}, and MAVEN~\citep{mahajan2019maven}.
All baselines are imported from PyMARL \cite{samvelyan19smac}. 
Detailed settings are  in  \emph{Appendix \ref{exp_detail}}.
Code is released in \url{https://github.com/QDPP-GitHub/QDPP}.
We consider four cooperative tasks in Fig. ~\ref{fig:mini_games}, all of  which require  non-trivial  value function decomposition to achieve the largest reward. 

\begin{figure}[t!]
     \centering
     \vspace{-8pt}
     \begin{subfigure}[r]{.21\textwidth}
         \centering
         \includegraphics[width=\textwidth]{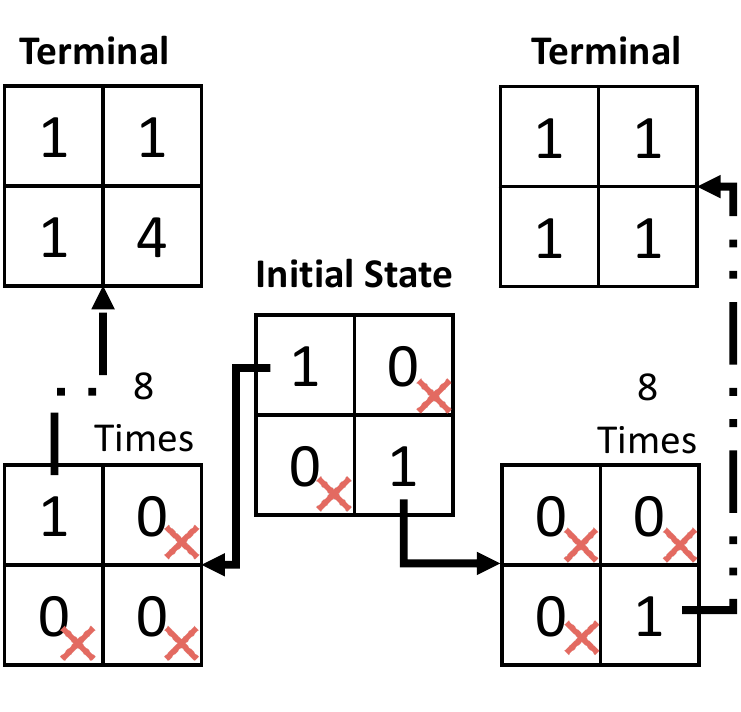}
         \vspace{-20pt}
         \caption{Pathological Stochastic Game}
         \label{fig:nmatrix_mini}
     \end{subfigure}
     \hspace{12pt}
     \begin{subfigure}[r]{.21\textwidth}
         \centering
         \includegraphics[width=\textwidth]{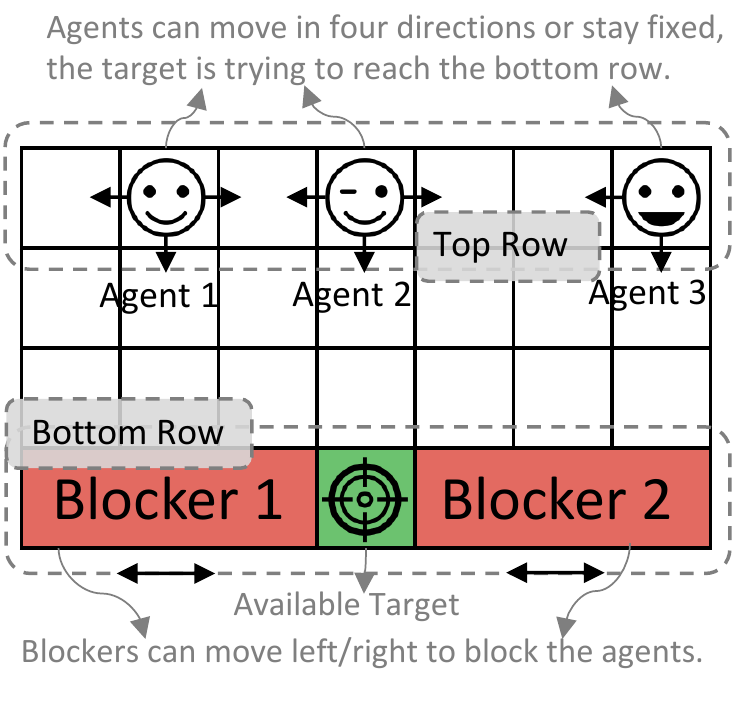}
                  \vspace{-20pt}
         \caption{Blocker Game}
         \label{fig:spread_mini}
     \end{subfigure}
     \\
     \begin{subfigure}[l]{.205\textwidth}
         \centering
         \includegraphics[width=\textwidth]{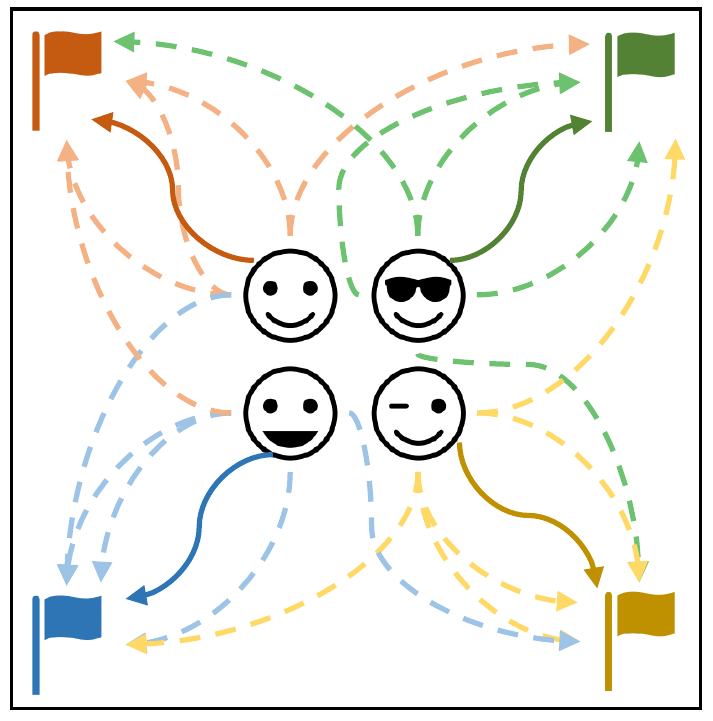}
                  \vspace{-15pt}
         \caption{Coordinated Navigation}
         \label{fig:blocker_mini}
     \end{subfigure}
  \hspace{15pt}
     \begin{subfigure}[l]{.205\textwidth}
         \centering
         \includegraphics[width=\textwidth]{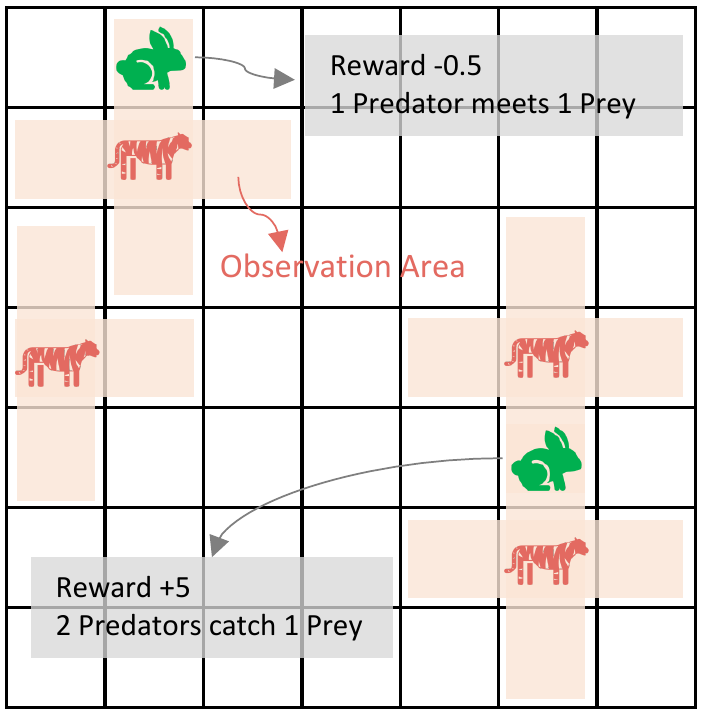}
                           \vspace{-15pt}
         \caption{Predator-Prey World}
         \label{fig:prey_mini}
     \end{subfigure}
     \vspace{-8pt}
     \caption{Multi-agent cooperative tasks. The size of the ground set for each task is a) $176$, b) $420$, c) $720$, d) $3920$.}
     \label{fig:mini_games}
\end{figure}


\begin{figure*}[t!]
\vspace{-5pt}
     \centering
     \begin{subfigure}[r]{.33\textwidth}
         \centering
         \includegraphics[width=\textwidth]{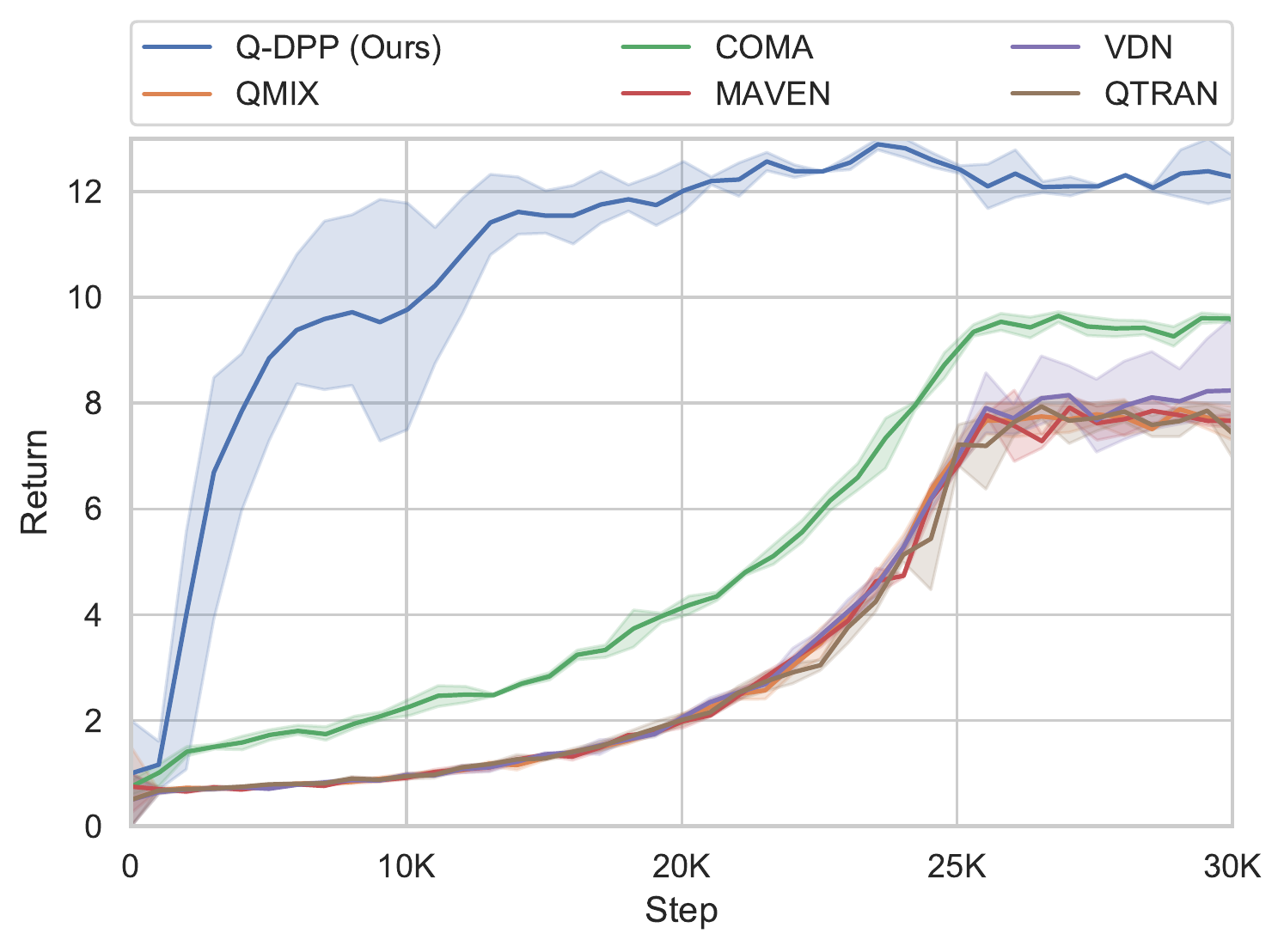}
         \vspace{-20pt}
         \caption{Multi-Step Matrix Game}
         \label{fig:nmatrix_mini_lc1}
     \end{subfigure}
     \begin{subfigure}[r]{.33\textwidth}
         \centering
         \includegraphics[width=\textwidth]{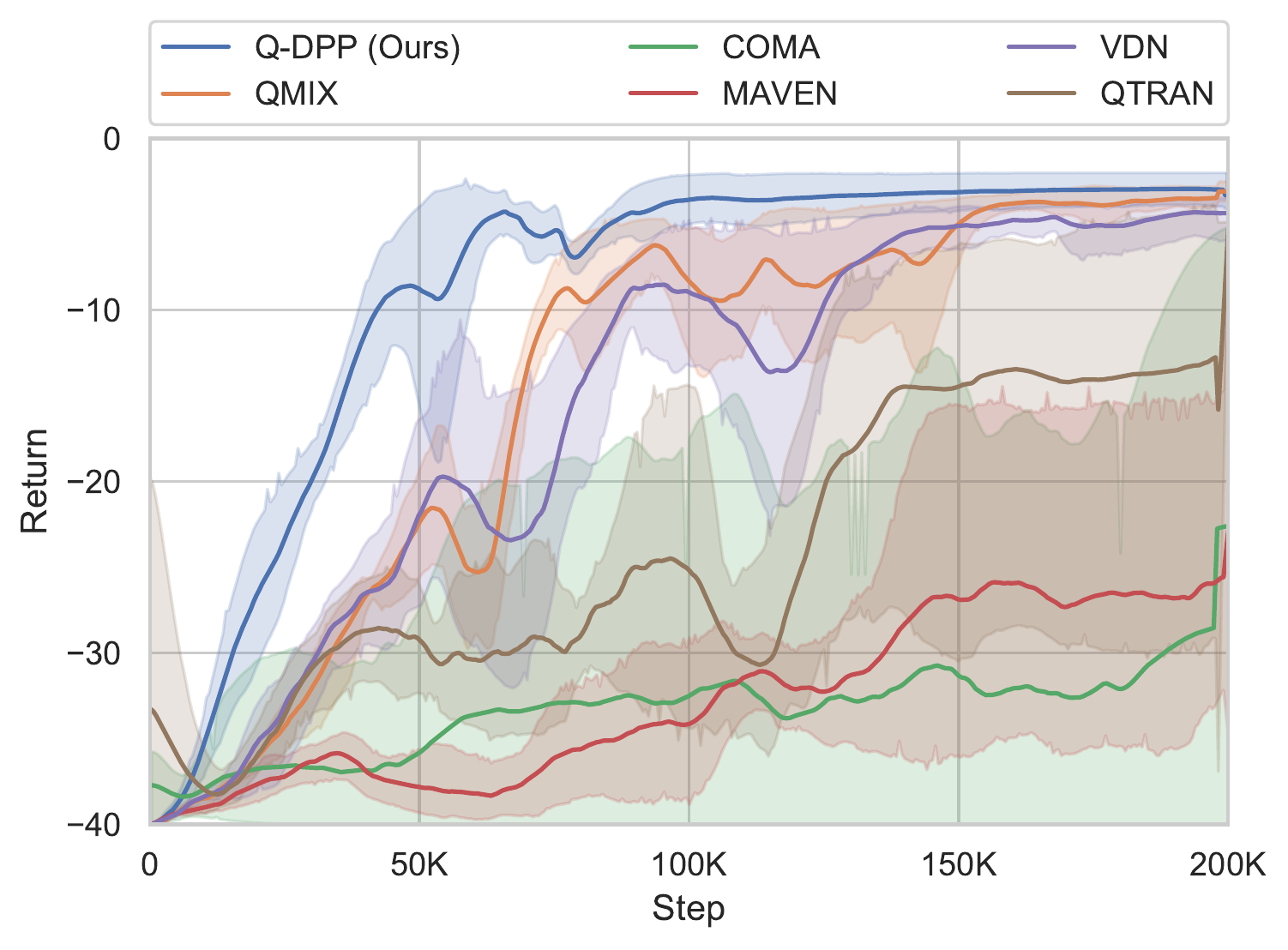}
                  \vspace{-20pt}
         \caption{Blocker Game}
         \label{fig:spread_mini_lc1}
     \end{subfigure}
     \begin{subfigure}[l]{.33\textwidth}
         \centering
         \includegraphics[width=\textwidth]{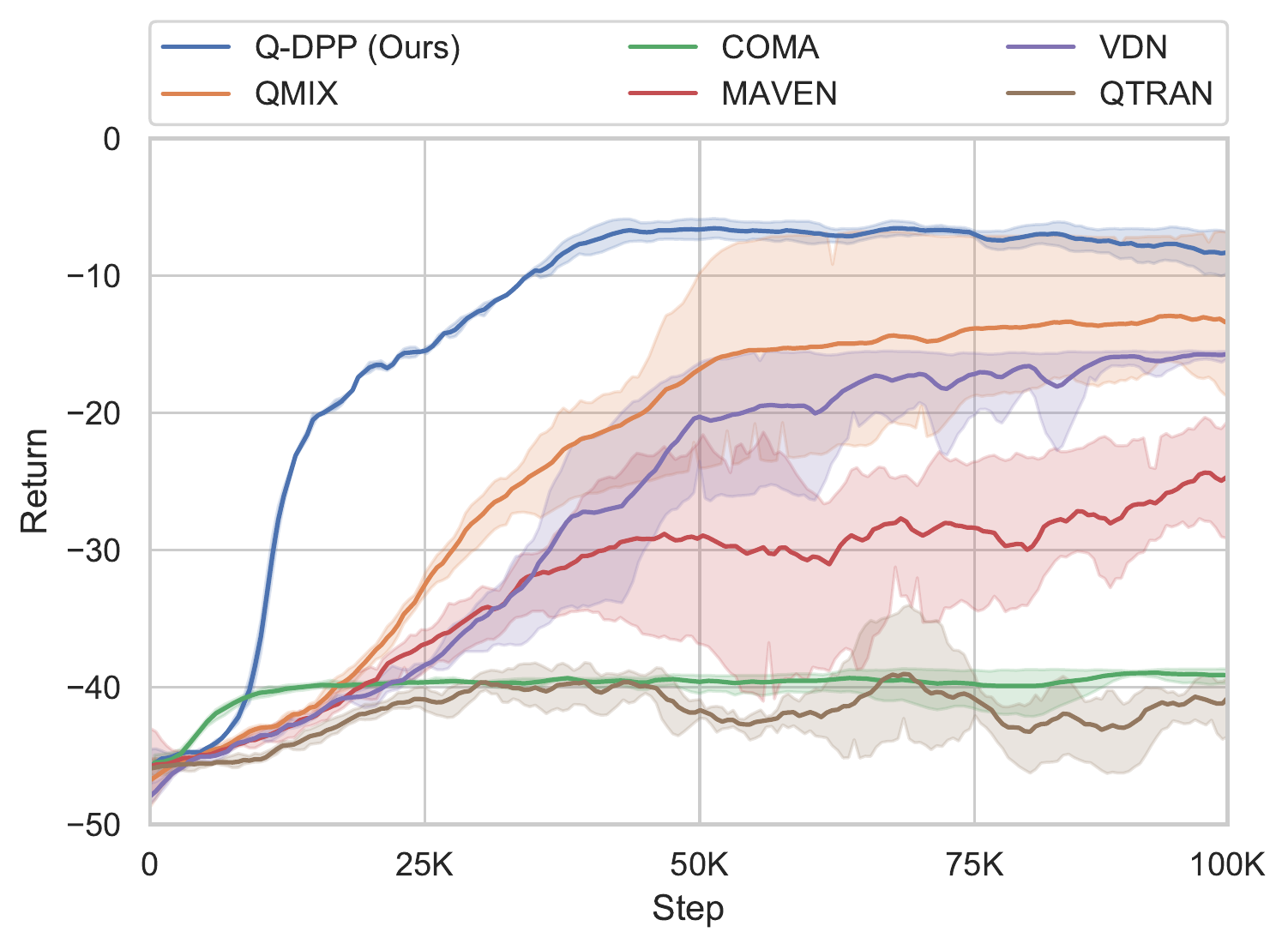}
                  \vspace{-20pt}
         \caption{Coordinated Navigation}
         \label{fig:blocker_mini_lc1}
     \end{subfigure} \\

    \begin{subfigure}[l]{.33\textwidth}
         \centering
         \includegraphics[width=\textwidth]{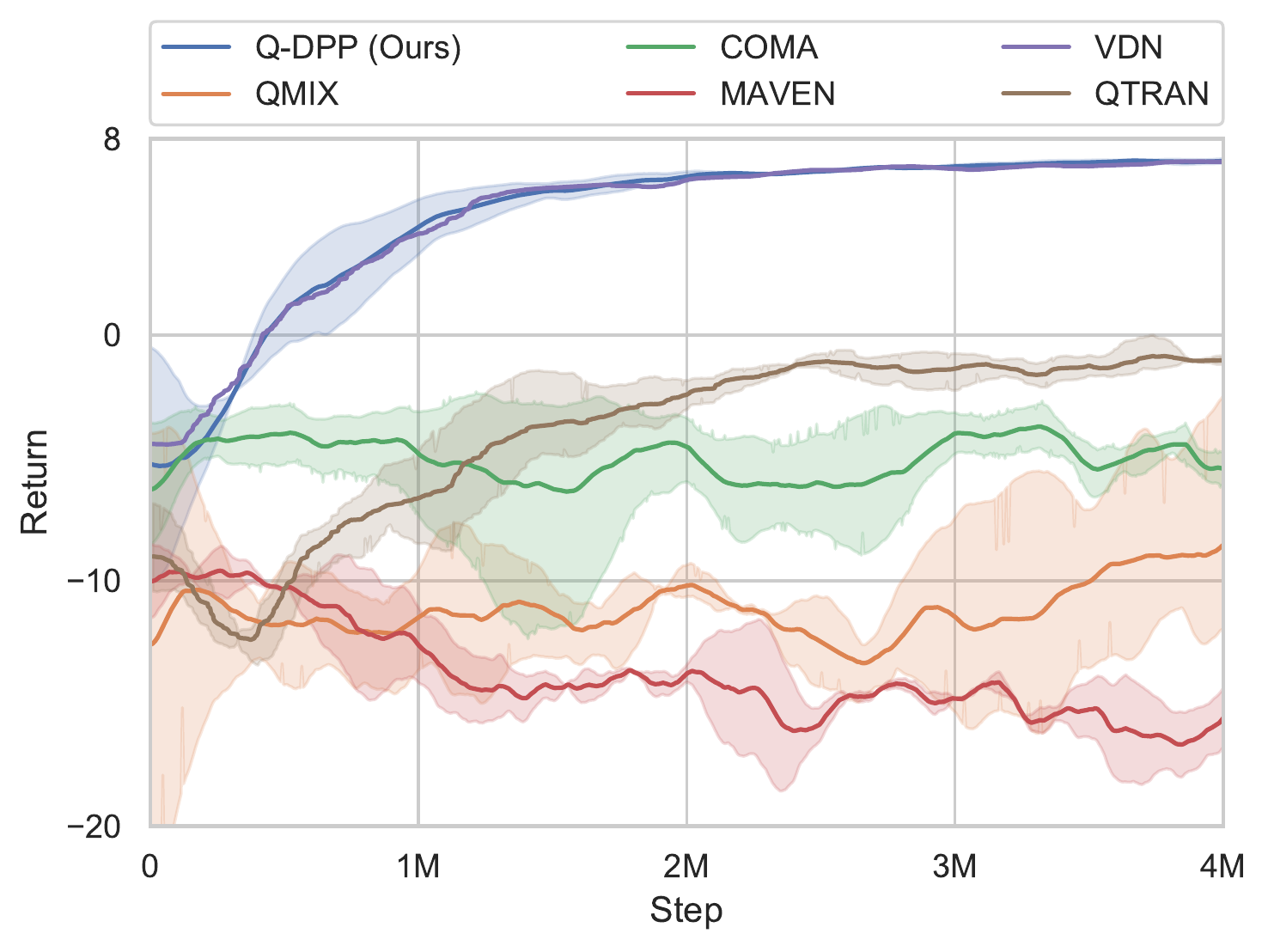}
                  \vspace{-20pt}
         \caption{Predator-Prey World}
         \label{fig:blocker_ablation_lc1}
     \end{subfigure}
     \begin{subfigure}[l]{.33\textwidth}
         \centering
         \includegraphics[width=\textwidth]{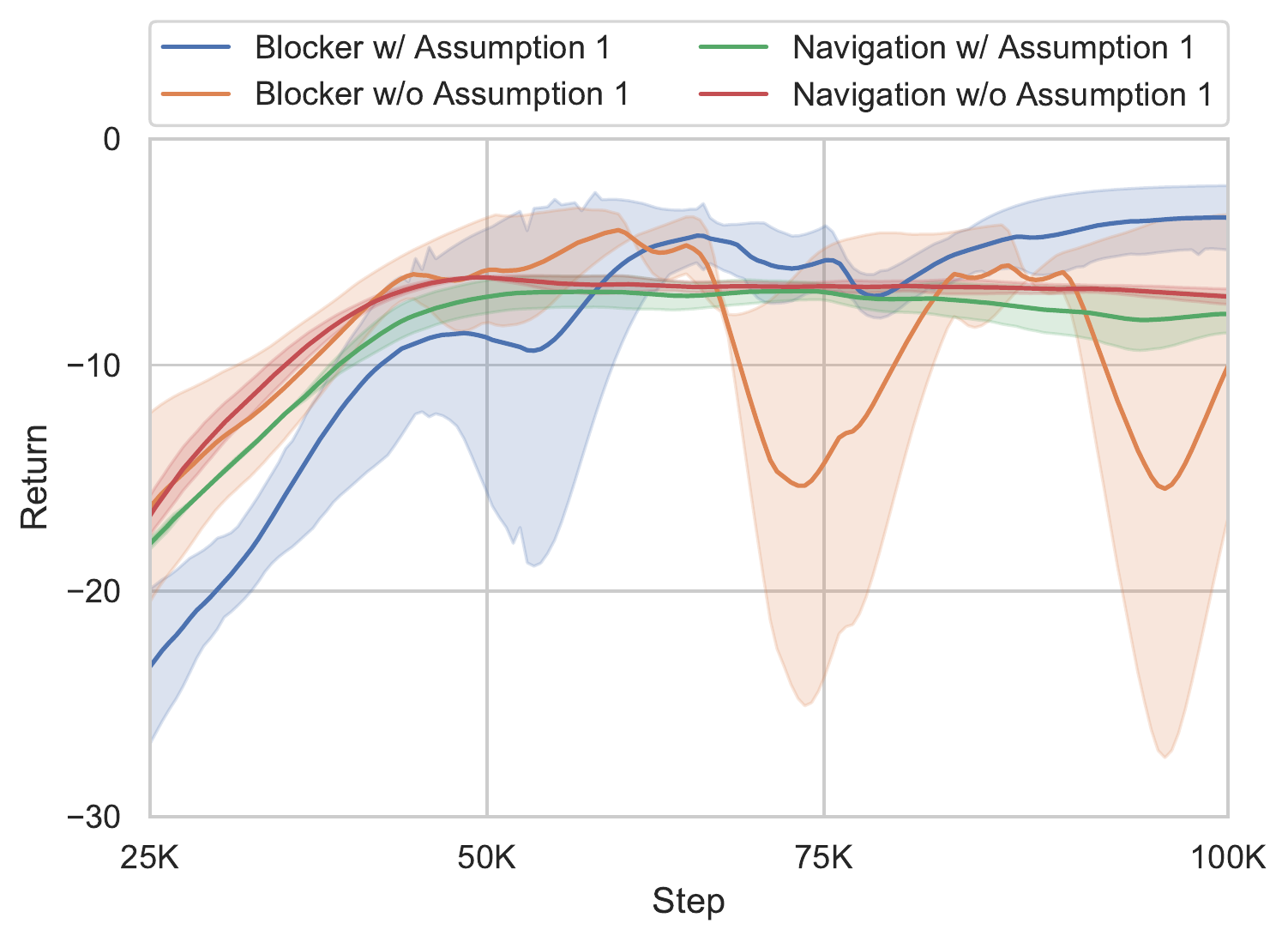}
                  \vspace{-20pt}
         \caption{Ablation study on Assumption \ref{assump:single}  }
         \label{fig:prey_mini_lc3}
     \end{subfigure}
          \begin{subfigure}[l]{.33\textwidth}
         \centering
         \includegraphics[width=\textwidth]{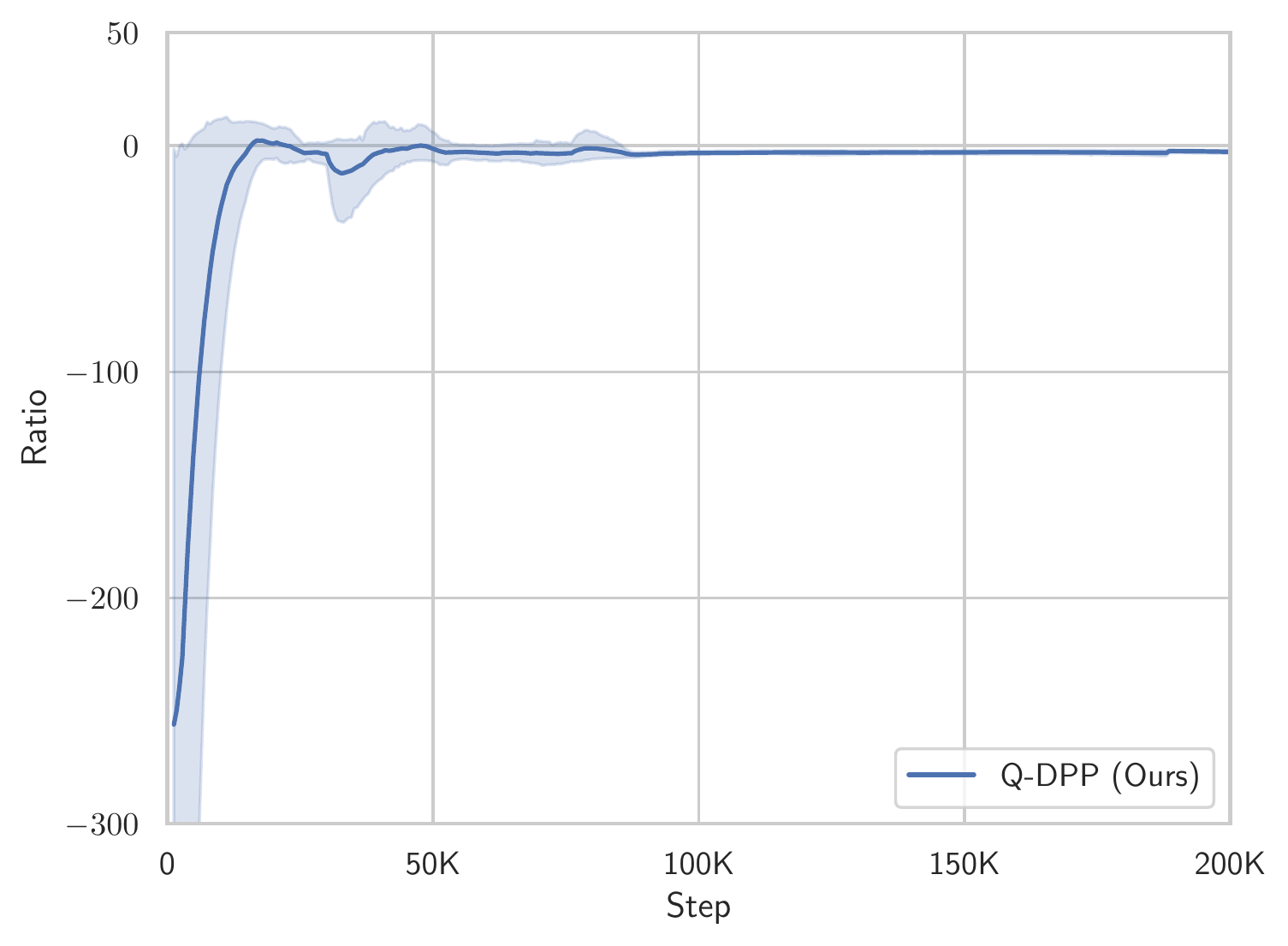}
                  \vspace{-20pt}
         \caption{Diversity / Quality Ratio}
         \label{fig:learned_pro}
     \end{subfigure}
          \vspace{-8pt}
     \caption{\textbf{(a)}-\textbf{(d)}:Performance  over time on different tasks. \textbf{(e)}: Ablation study on Assumption \ref{assump:single} on Blocker game.  \textbf{(f)}: The ratio of diversity to quality, i.e.,  $\log \det(\bm{\mathcal{B}}_{Y}^{\top}\bm{\mathcal{B}}_Y) / \sum_{i=1}^{N}Q_{\mathcal{I}(o_i, a_i)}(o_i, a_i)$, during training on Blocker game.}
     \label{fig:mini_games_lc}
     \vspace{-5pt}
\end{figure*}

\begin{figure*}[t!]
\vspace{0pt}
     \centering
      \hspace{10pt}
     \begin{subfigure}[r]{.3\textwidth}
         \centering
         \includegraphics[width=\textwidth]{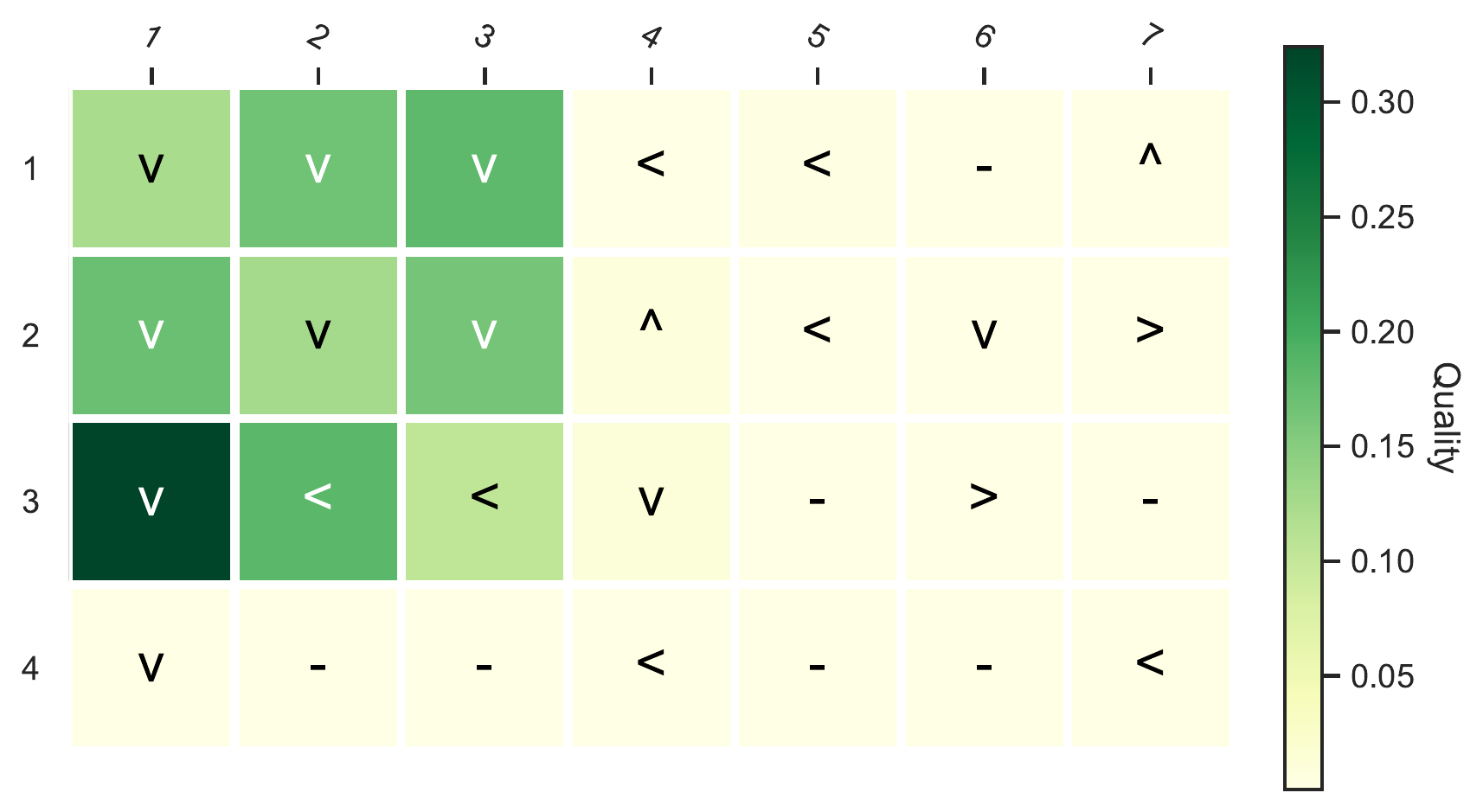}
         \vspace{-15pt}
         \caption{Agent 1.}
         \label{fig:quality_blocker_agent1}
     \end{subfigure}
     \hspace{10pt}
     \begin{subfigure}[r]{.3\textwidth}
         \centering
         \includegraphics[width=\textwidth]{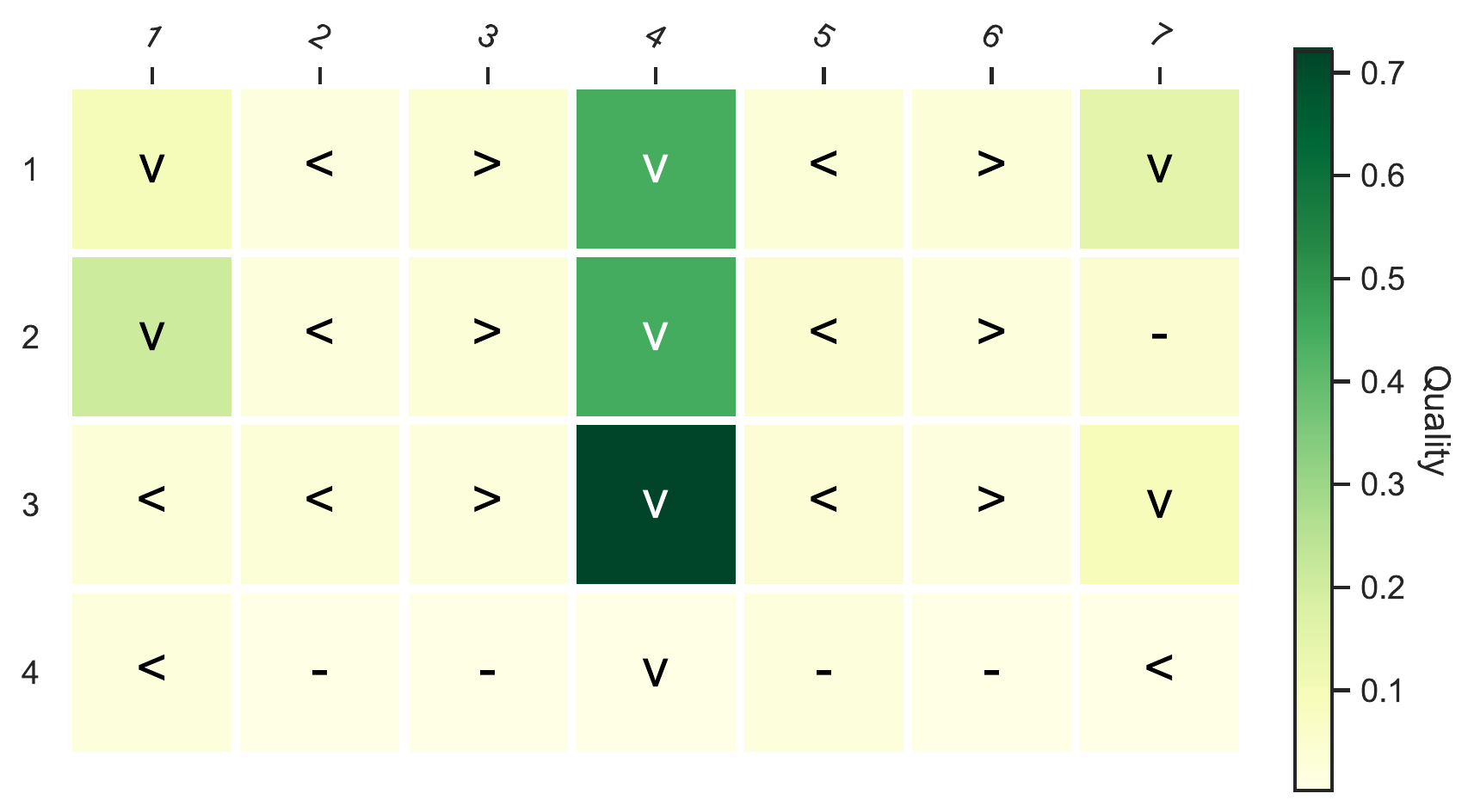}
                  \vspace{-15pt}
         \caption{Agent 2.}
         \label{fig:quality_blocker_agent2}
     \end{subfigure}
          \hspace{10pt}
      \begin{subfigure}[r]{.3\textwidth}
         \centering
         \includegraphics[width=\textwidth]{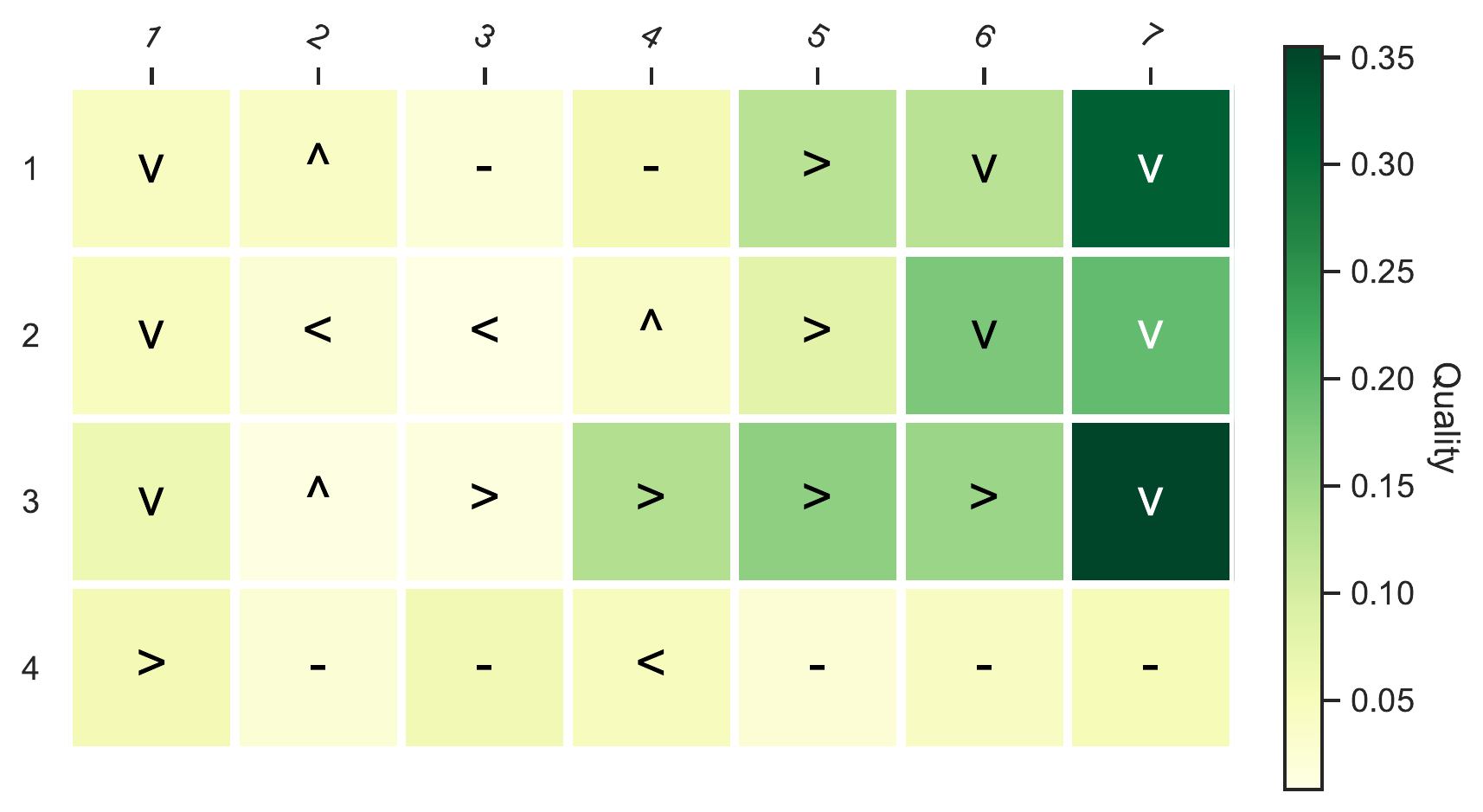}
                  \vspace{-15pt}
         \caption{Agent 3.}
         \label{fig:quality_blocker_agent3}
     \end{subfigure}
              \vspace{-8pt}
     \caption{\textbf{(a)}-\textbf{(c)}: Each of the agent's decentralized policy, i.e., $\arg \max_a Q_i(o_i, a) $, during execution on Blocker game. }
     \label{fig:quality_blocker_agents}
                   \vspace{-15pt}
\end{figure*}

\textbf{Pathological Stochastic Game}.
The optimal policy of this game is to let both agents keep acting top left  until the $10$-th step to change to bottom right, which results in the optimal  reward of $13$. 
The design of such stochastic game intends to be pathological. First, it is non-monotonic (thus QMIX surely fails), second, it demonstrates  \emph{relative overgeneralization} \cite{wei2018multiagent}  because both agents playing the 1st action on average offer a higher reward $10$ when matched with arbitrary actions from the other agent.  
We allow agent to observe the current step number and the joint action in the last time-step. Zero reward leads to immediate termination.
Fig.~\ref{fig:nmatrix_mini_lc1} shows Q-DPP can  converge to the global optimal in only $20K$ steps while  other baselines struggle.

\textbf{Blocker Game} \& \textbf{Coordinated Navigation}.
Blocker game ~\citep{heess2012actor} requires   agents to reach the bottom row by coordinating with its teammates to deceive the blockers that  can move left/right to block them. 
The  navigation game requires  four agents  to reach four different landmarks. For both tasks, it costs all  agents  $-1$ reward per  time-step before they all  reach the destination. 
Depending on the starting points, the largest reward of the game are $-3$ and $-6$ respectively. 
Both tasks are challenging in the sense that  coordination is rather challenging for  agents that only have  decentralized policies and local observations. 
Fig.~\ref{fig:spread_mini_lc1} \&~\ref{fig:blocker_mini_lc1} suggest Q-DPP still achieves the best performance. 

\textbf{Predator-Prey World}.
In this task, four predators attempt to capture two randomly-moving preys. 
Each predator can move in four directions but they only have local views.
The predators get a team reward of $1$ if two or more  predators are capturing the same prey at the same time, and they are penalized for $-0.5$ if only one of them captures a prey. The game terminates when all preys are caught.
Fig.~\ref{fig:blocker_ablation_lc1} shows Q-DPP's superior performance than all other baselines.

Apart from the best performance in terms of rewards, here we offer more insights of why and how Q-DPP works well. 

\textbf{The Importance of Assumption~\ref{assump:single}.} 
Assumption~\ref{assump:single} is the premise for the correctness of Q-DPP sampler to hold.  
To investigate its impact in practice, we conduct the ablation study on Blocker and Navigation games. 
We implement such assumption via an auxiliary loss function of $\max (0, \sigma_j^2 - \hat{\sigma}_{i, j}^2/\delta)$ that penalizes  the violation of the assumption, we set $\delta=0.5$. 
Fig.~\ref{fig:prey_mini_lc3} presents the performance comparisons of the Q-DPPs with and without such additional loss function. We can tell that 
maintaining such a condition, though not helping improve the performance, stablizes the training process by significantly reducing the variance of the rewards. 
We believe this is because violating  Assumption \ref{algo:main_algo1}  leads to over-estimating the probability of certain observation-action pairs in the partition where the violation happens, such over-estimation can make the  agent stick to a poor local observation-action pair for some time.

\textbf{The Satisfaction of Eq.~\ref{eq:igm}.} We show empirical evidence on Blocker game that the natural factorization that Q-DPP offers indeed satisfy Eq.~\ref{eq:igm}.   
Intuitively, Q-DPPs encourage agents to acquire diverse behavorial models during training so that the optimal action of one agent  does not depend on the actions of the other agents  during the decentralized execution stage,    as a result, Eq.~\ref{eq:igm} can be satisfied. 
Fig.~\ref{fig:quality_blocker_agents} (a-c) justify such intuition by showing  Q-DPP  learns mutually orthogonal behavioral models. Given  the distinction among agents' individual policies, one can tell that the joint optimum is reached through individual optima.

\textbf{Quality versus Diversity.} 
We investigate the change of the relative importance of quality versus diversity during training. 
On Blocker game, we show the ratio  of $\log \det\big(\bm{\mathcal{B}}_{Y}^{\top}\bm{\mathcal{B}}_Y\big) / \sum_{i=1}^{N}Q_{\mathcal{I}(o_i, a_i)}\big(o_i, a_i\big)$, which reflects how the learning algorithm balances maximizing reward  against encouraging diverse behaviors. 
In Fig.~\ref{fig:learned_pro}, we can see that the ratio gradually converges to $0$. The diversity term plays a less important role with the development of training; this is also expected since explorations tend to be rewarded more  at the early stage of a task.

\section{Conclusion}

We proposed Q-DPP, a new type of value-function approximator for cooperative multi-agent reinforcement learning.
Q-DPP, as a probabilistic way of modeling sets, considers not only the quality of  agents' actions  towards reward maximization, but the  diversity of agents' behaviors as well. We have demonstrated  that Q-DPP addresses the limitation of current major solutions including VDN, QMIX, and QTRAN by  learning the  value function decomposition without structural constraints. In the future, we plan to  investigate other kernel representations for Q-DPPs to tackle the  tasks with continuous states and continuous actions. 

\clearpage

\section*{Acknowledgement}

We sincerely thank Dr. Haitham Bou Ammar for his constructive comments. 

\bibliography{references}
\bibliographystyle{icml2020}

\clearpage

\appendix
\clearpage

\section{Detailed Proofs}
\setcounter{theorem}{0}
\setcounter{proposition}{0}

\subsection{Proof of Proposition 1}
\label{app:prop1}

\begin{proposition}[Volume preservation of Gram-Schmidt, see Chapter 7 in \citet{shafarevich2012linear}, also Lemma 3.1 in \citet{celis2018fair}.]
Let $\mathcal{U}_i = \operatorname{span}\{\bm{w}_1,\ldots,\bm{w}_{i-1}\}$ and $\bm{w}_i \in \mathbb{R}^P$ be the $i$-th  row of   $\bm{\mathcal{W}} \in \mathbb{R}^{M\times P}$,  then $  \prod_{i=1}^{M}\|\amalg_{\mathcal{U}_i}(\bm{w}_i)\|^2 = \det(\bm{\mathcal{W}}\bm{\mathcal{W}}^{\top})$. 
\label{prop1:gramschmidt}
\end{proposition}
\begin{proof}
Such property has been mentioned in linear algebra textbook, e.g.,  Chapter 7 in \citet{shafarevich2012linear}. \citet{celis2018fair} also gave out a proof by induction\footnote{We believe their proof is a special case, as  interchanging the order of  rows can  actually change the determinant value, i.e., $\det (W W^{\top}) \neq \Big[\begin{array}{l}
{w_{k}} \\
{W^{\prime}}
\end{array}\Big]\Big[\begin{array}{ll}
{w_{k}^{\top}} & {W^{\prime \top}}
\end{array}\Big]$ where the row vectors are denoted as $W = \{w_1,\ldots, w_k \}$ and $W' = \{w_1, \ldots, w_{k-1} \}$.} in  Lemma 3.1.
Here we provide our own intuition  of such property through the classical Gaussian elimination method.

We first define  an orthogonalization operator $\sqcap_{\bm{w}_i}(\bm{w}_j)$ that takes an input of a vector $\bm{w}_j \in \mathbb{R}^P$ and outputs another vector that is orthogonal to a given vector $\bm{w}_i \in \mathbb{R}^P$ by 
\begin{equation}
\sqcap_{\bm{w}_i}(\bm{w}_j):= \bm{w}_j - \bm{w}_i  \langle\bm{w}_i, \bm{w}_j\rangle/\|\bm{w}_i\|^2 \ .
\label{eq:proj_def}    
\end{equation}
Based on the Eq. \ref{eq:proj_def}, we know that $\forall \bm{w}_i,\bm{w}_j,\bm{w}_k \in \mathbb{R}^P, $ $$\sqcap_{\bm{w}_i}(\bm{w}_j + \bm{w}_k) = \sqcap_{\bm{w}_i}(\bm{w}_j) + \sqcap_{\bm{w}_i}(\bm{w}_k). $$
Besides, we have two properties for the orthogonalization operator that will be used later;  we present as lemmas.  

\begin{lemma}[Change of Projection Base] Let $\bm{w}_i$, $\bm{w}_j$, $\bm{w}_k \in \mathbb{R}^P$, we have $\bm{w}_j \cdot {\sqcap_{\bm{w}_i}(\bm{w}_k)}^{\top}  = {\sqcap_{\bm{w}_i}(\bm{w}_j)} \cdot {\sqcap_{\bm{w}_i}(\bm{w}_k)}^{\top} $. 
\label{lemma:prj}
\end{lemma}
\begin{proof}
Based the definition of Eq.~\ref{eq:proj_def}, one can easily write that the left hand side equals to the right hand side.
\end{proof}

\begin{lemma}[Subspace Orthogonalization] Let $\bm{w}_i$, $\bm{w}_j$, $\bm{w}_k \in \mathbb{R}^P$, we have ${\sqcap_{{\sqcap_{\bm{w}_i}(\bm{w}_j)} }(\bm{w}_k)\cdot {\sqcap_{\bm{w}_i}(\bm{w}_k)}^{\top}}= \big\|{\amalg_{\mathcal{U}_k}(\bm{w}_k)}\big\|^2 $ where $\mathcal{U}_k = \operatorname{span} \{\bm{w}_i, \bm{w}_j\}$. 
\label{lemma:prjsub}
\end{lemma}
\begin{proof}
The left-hand side of equation can be written by 
\begin{equation}
\begin{aligned}[b]
&{\sqcap_{{\sqcap_{\bm{w}_i}(\bm{w}_j)} }(\bm{w}_k)\cdot {\sqcap_{\bm{w}_i}(\bm{w}_k)}^{\top}}  \\ =&  \Big(\bm{w}_k - {\sqcap_{\bm{w}_i}(\bm{w}_j)} \dfrac{\langle {\sqcap_{\bm{w}_i}(\bm{w}_j)}, \bm{w}_k  \rangle} { \|{\sqcap_{\bm{w}_i}(\bm{w}_j)} \|^2)} \Big) \cdot \Big(\bm{w}_k - \bm{w}_i  \dfrac{\langle\bm{w}_i, \bm{w}_k\rangle} {\|\bm{w}_i\|^2} \Big)^{\top} \\
=& \   \bm{w}_k \bm{w}_k^{\top} - \dfrac{{\sqcap_{\bm{w}_i}(\bm{w}_j)}\cdot \bm{w}_k^{\top}}{\|{\sqcap_{\bm{w}_i}(\bm{w}_j)}\|} \langle \dfrac{{\sqcap_{\bm{w}_i}(\bm{w}_j)}}{{\|{\sqcap_{\bm{w}_i}(\bm{w}_j)}\|}}, \bm{w}_k  \rangle \\ & \ \ \ \ \ \ \ \ \ \ \ \ \ \ \ \  -  \dfrac{\bm{w}_k \cdot \bm{w}_i^{\top}}{\|\bm{w}_i\|}  \langle \dfrac{\bm{w}_i^{\top}}{\|\bm{w}_i\|}, \bm{w}_k \rangle \  \ .
\label{lemma:left}
\end{aligned}
\end{equation}
On the other hand, ${\amalg_{\mathcal{U}_k}(\bm{w}_k)}$ represents the orthogonal projection of $\bm{w}_k$ to the subspace that is spanned by $\bm{w}_i \text{ and } \bm{w}_j$. 
Since $\Big\{\frac{\bm{w}_i}{\|\bm{w}_i\|}, \frac{\sqcap_{\bm{w}_i}(\bm{w}_j)}{\|{\sqcap_{\bm{w}_i}(\bm{w}_j)} \|} \Big\}$ form a set of orthornormal basis for the subspace  $\mathcal{U}_k =  \operatorname{span} \{\bm{w}_i, \bm{w}_j\}$, according to the definition of $\amalg_{\mathcal{U}_k}(\bm{w}_k)$ in Section  \ref{sec:sample}, we can write ${\amalg_{\mathcal{U}_k}(\bm{w}_k)}$ as $\bm{w}_k$ minus the projection of $\bm{w}_k$ on the subspace that is spanned by $\Big\{\frac{\bm{w}_i}{\|\bm{w}_i\|}, \frac{\sqcap_{\bm{w}_i}(\bm{w}_j)}{\|{\sqcap_{\bm{w}_i}(\bm{w}_j)} \|} \Big\}$, i.e., 
\begin{equation}
\begin{aligned}[b]
{\amalg_{\mathcal{U}_k}(\bm{w}_k)}=\bm{w}_k \ \  - \ \  &\langle \bm{w}_k, \frac{\bm{w}_i}{\|\bm{w}_i\|} \rangle \frac{\bm{w}_i}{\|\bm{w}_i\|}   \\ \ \  - \ \ & \langle \bm{w}_k, \frac{\sqcap_{\bm{w}_i}(\bm{w}_j)}{\|{\sqcap_{\bm{w}_i}(\bm{w}_j)} \|} \rangle \frac{\sqcap_{\bm{w}_i}(\bm{w}_j)}{\|{\sqcap_{\bm{w}_i}(\bm{w}_j)} \|}. 
\label{lemma:right}
\end{aligned}
\end{equation}
Under the orthonormal property of $\frac{\bm{w}_i}{\|\bm{w}_i\|} \cdot \frac{\sqcap_{\bm{w}_i}(\bm{w}_j)}{\|{\sqcap_{\bm{w}_i}(\bm{w}_j)} \|}^{\top}=0$ and $\Big\|\frac{\bm{w}_i}{\|\bm{w}_i\|}\Big\|^2 = \Big\|\frac{\sqcap_{\bm{w}_i}(\bm{w}_j)}{\|{\sqcap_{\bm{w}_i}(\bm{w}_j)}\|}\Big\|^2 =1$, 
 finally,  squaring the Eq.~\ref{lemma:right} from both sides  leads us to the  Eq.~\ref{lemma:left}, i.e., $\big\|{\amalg_{\mathcal{U}_k}(\bm{w}_k)}\big\|^2  = {\sqcap_{{\sqcap_{\bm{w}_i}(\bm{w}_j)} }(\bm{w}_k)\cdot {\sqcap_{\bm{w}_i}(\bm{w}_k)}^{\top}}$.   
\end{proof}

Assuming $\{\bm{w}_1, \ldots, \bm{w}_M\}$ being the rows of $\bm{\mathcal{W}}$, then applying the Gram-Schmidt orthogonalization process gives  
\[
 \text{Gram-Schmidt}\Big(\big\{\bm{w}_i\big\}_{i=1}^{M}\Big) = \Big\{{\amalg_{\mathcal{U}_i}(\bm{w}_i)} \Big\}_{i=1}^{M} \] 
where $\mathcal{U}_i = \operatorname{span}\{\bm{w}_1,\ldots,\bm{w}_{i-1}\}$.
Note that we don't consider  normalizing each ${\amalg_{\mathcal{U}_i}(\bm{w}_i)}$  in this work.

In fact,  the effect on the Gram matrix determinant  $\det ({\bm{\mathcal{W}}}{\bm{\mathcal{W}}}^{\top})$ of applying the Gram-Schmidt process on the rows of $\bm{\mathcal{W}}$  is equivalent to applying  Gaussian elimination \cite{noble1988applied} to transform the Gram matrix to be upper triangular. 
Since adding a row/column of a matrix multiplied by a scalar  to another row/column of that matrix  will not change the determinant value of the original matrix \cite{noble1988applied}, Gaussian elimination, so as the Gram-Schmidt process, preserves the determinant.

To  illustrate the above equivalence, we demonstrate the Gaussian elimination process step-by-step on the case of $M=3$,  the determinant of such a Gram matrix is  
\begin{equation}
\det \big({\bm{\mathcal{W}}}{\bm{\mathcal{W}}}^{\top} \big)=\det \begin{pmatrix}
\bm{w}_1\bm{w}_1^{\top} & \bm{w}_1\bm{w}_2^{\top} &  \bm{w}_1\bm{w}_3^{\top} \\ 
\bm{w}_2\bm{w}_1^{\top} & \bm{w}_2\bm{w}_2^{\top} & \bm{w}_2\bm{w}_3^{\top}\\ 
\bm{w}_3\bm{w}_1^{\top} & \bm{w}_3\bm{w}_2^{\top} &  \bm{w}_3\bm{w}_3^{\top} 
\end{pmatrix} .
\label{eq:gram}
\end{equation}

To apply Gaussian elimination to turn the Gram matrix to be upper triangular, first, we  multiply the $1$-st row by $-\frac{\bm{w}_2\bm{w}_1^{\top}}{\bm{w}_1\bm{w}_1^{\top}}$ and then  add the result to the $2$-nd row; without affecting the determinant, we have the $2$-nd row transformed into 
\begin{equation}
\begin{aligned}[b]
&\Big[0, \bm{w}_2\bm{w}_2^{\top} -  \frac{\bm{w}_2\bm{w}_1^{\top}}{\bm{w}_1\bm{w}_1^{\top}}\bm{w}_1\bm{w}_2^{\top}, \bm{w}_2\bm{w}_3^{\top} -  \frac{\bm{w}_2\bm{w}_1^{\top}}{\bm{w}_1\bm{w}_1^{\top}}\bm{w}_1\bm{w}_3^{\top} \Big] \\
=& \Big[0, \bm{w}_2 \cdot {\sqcap_{\bm{w}_1}(\bm{w}_2)}^{\top}, \bm{w}_3 \cdot {\sqcap_{\bm{w}_1}(\bm{w}_2)}^{\top} \Big] \\
=& \Big[0, {\sqcap_{\bm{w}_1}(\bm{w}_2)} \cdot {\sqcap_{\bm{w}_1}(\bm{w}_2)}^{\top}, {\sqcap_{\bm{w}_1}(\bm{w}_3)} \cdot {\sqcap_{\bm{w}_1}(\bm{w}_2)}^{\top} \Big]. \ \ (\text{\emph{Lemma \ref{lemma:prj}}})
\label{eq:2ndrow}
\end{aligned}    
\end{equation}
Similarly, we can apply the same process on the $3$-rd row, which can be written as  
\begin{equation}
\begin{aligned}
&\Big[0, \bm{w}_3\bm{w}_2^{\top} -  \frac{\bm{w}_2\bm{w}_1^{\top}}{\bm{w}_1\bm{w}_1^{\top}}\bm{w}_1\bm{w}_2^{\top}, \bm{w}_3\bm{w}_3^{\top} -  \frac{\bm{w}_2\bm{w}_1^{\top}}{\bm{w}_1\bm{w}_1^{\top}}\bm{w}_1\bm{w}_3^{\top} \Big] \\
=& \Big[0, {\sqcap_{\bm{w}_1}(\bm{w}_2)} \cdot {\sqcap_{\bm{w}_1}(\bm{w}_3)}^{\top}, {\sqcap_{\bm{w}_1}(\bm{w}_3)} \cdot {\sqcap_{\bm{w}_1}(\bm{w}_3)}^{\top} \Big] .
\label{eq:3rdrow}
\end{aligned}    
\end{equation}
To make $\bm{\mathcal{W}}\bm{\mathcal{W}}^{\top}$ upper triangular, we need to make the $2$-nd element in the $3$-rd row be zero.
To achieve that, we multiply $-\frac{{\sqcap_{\bm{w}_1}(\bm{w}_2)} \cdot {\sqcap_{\bm{w}_1}(\bm{w}_3)}^{\top}}{{\sqcap_{\bm{w}_1}(\bm{w}_2)} \cdot {\sqcap_{\bm{w}_1}(\bm{w}_2)}^{\top}}$ to Eq. \ref{eq:2ndrow} and add the multiplication to Eq. \ref{eq:3rdrow}, and the $3$-rd row can be further transformed into 
\begin{equation}
\begin{aligned}
& \Big[0, 0,   {\sqcap_{\bm{w}_1}(\bm{w}_3)}  \cdot {\sqcap_{\bm{w}_1}(\bm{w}_3)}^{\top}  - \\  & \ \ \ \ \ \ \ \ \ \ \ \ \ \ \ \ \  \frac{{\sqcap_{\bm{w}_1}(\bm{w}_2)} \cdot {\sqcap_{\bm{w}_1}(\bm{w}_3)}^{\top}}{{\sqcap_{\bm{w}_1}(\bm{w}_2)} \cdot {\sqcap_{\bm{w}_1}(\bm{w}_2)}^{\top}} {\sqcap_{\bm{w}_1}(\bm{w}_3)} \cdot {\sqcap_{\bm{w}_1}(\bm{w}_2)}^{\top} \Big] \\
= & \Big[0, 0,  \sqcap_{\bm{w}_1}{(\bm{w}_3)} \cdot \sqcap_{\sqcap_{\bm{w}_1}(\bm{w}_2)}\big(\sqcap_{\bm{w}_1}(\bm{w}_3) \big)^{\top}\Big] \\
= & \Big[0, 0,  \sqcap_{\bm{w}_1}{(\bm{w}_3)} \cdot \sqcap_{\sqcap_{\bm{w}_1}(\bm{w}_2)}\Big(\bm{w}_3 - \bm{w}_1  \dfrac{\langle\bm{w}_1, \bm{w}_3\rangle} {\|\bm{w}_1\|^2}  \Big)^{\top}\Big] \\
= & \Big[0, 0,  \sqcap_{\bm{w}_1}{(\bm{w}_3)} \cdot \Big(\sqcap_{\sqcap_{\bm{w}_1}(\bm{w}_2)}(\bm{w}_3)  \\ & \ \ \ \ \ \ \ \ \ \ \ \ \ \ \ \ \ \ \ \ \ \ \ \ \ \ \ \ \ \ \ \ \ \ \ \ \ \ \ \ \ \ -  \sqcap_{\sqcap_{\bm{w}_1}(\bm{w}_2)} \big(\bm{w}_1  \dfrac{\langle\bm{w}_1, \bm{w}_3\rangle} {\|\bm{w}_1\|^2}\big)  \Big)^{\top}\Big] \\
= & \Big[0, 0,  \sqcap_{\bm{w}_1}{(\bm{w}_3)} \cdot \sqcap_{\sqcap_{\bm{w}_1}(\bm{w}_2)}\big(\bm{w}_3\big)^{\top}\Big]   \\
= & \Big[0, 0,  \big\|\amalg_{\mathcal{U}_3}(\bm{w}_3) \big\|^2 \Big]. \ \ \ \ \ \ \ \ \ \ \ \   (\text{\emph{Lemma  \ref{lemma:prjsub}}})
\label{eq:new3rdrow}
\end{aligned}    
\end{equation}
In the fourth equation of Eq.\ref{eq:new3rdrow}, we use the property that $\sqcap_{\bm{w}_1}(\cdot)\cdot \sqcap_{\sqcap_{\bm{w}_1}(\cdot)} (\bm{w}_1)^{\top} =0 $, i.e., the inner product between a vector and its own   orthogonalization equals to zero. 

Given the Gran matrix is now upper triangular, by 
putting Eq. \ref{eq:2ndrow} and Eq. \ref{eq:new3rdrow} into Eq. \ref{eq:gram}, and define $\mathcal{U}_1 = \emptyset, \mathcal{U}_2=\{ \bm{w}_1\}, \mathcal{U}_3=\{ \bm{w}_1, \bm{w}_2\}$,  we can write the determinant  to be \begin{equation}
\begin{aligned}
&\det \big({\bm{\mathcal{W}}}{\bm{\mathcal{W}}}^{\top} \big) \\ 
= \det & \begin{pmatrix}
\bm{w}_1\bm{w}_1^{\top} & \bm{w}_1\bm{w}_2^{\top} &  \bm{w}_1\bm{w}_3^{\top} \\ 
0 & \big\|{\sqcap_{\bm{w}_1}(\bm{w}_2)}\big\|^2 & {\sqcap_{\bm{w}_1}(\bm{w}_3)} \cdot {\sqcap_{\bm{w}_1}(\bm{w}_2)}^{\top}\\ 
0 & 0 &   \big\|\amalg_{\mathcal{U}_3}(\bm{w}_3) \big\|^2
\end{pmatrix}  \\
=  \prod_{i=1}^{3}&\Big\|\amalg_{\mathcal{U}_i}(\bm{w}_i)\Big\|^2 \ \ \ . \nonumber
\end{aligned} 
\end{equation}

When $M \ge 3$, the consequence of 
eliminating all $j$-th elements ($j<i$) in the $i$-th row of the Gram matrix  $\bm{\mathcal{W}}\bm{\mathcal{W}}^{\top}_{(i,j)}$ by   Gaussian elimination   is equivalent to the $i$-th step of the Gran-Schmidt process applied  on the  vector set $\{\bm{w}_i\}_{i=1}^{M}$, in other words,  the $(i, i)$-th element of the Gram matrix after Gaussian elimination is essentially the squared norm of $\amalg_{\mathcal{U}_i}(\bm{w}_i)$. 
Finally, since the determinant of an upper-triangular matrix is simply the multiplication of its diagonal elements, we have  $\prod_{i=1}^{M}\big\|\amalg_{\mathcal{U}_i}(\bm{w}_i)\big\|^2.$
\end{proof}

\newpage

\subsection{Proof of Theorem 1}

\begin{theorem}[Approximation Guarantee of Orthogonalizing Sampler]
For a Q-DPP defined in Definition \ref{def:dpp}, under Assumption \ref{assump:single}, the Orthogonalizing Sampler described in Algorithm \ref{algo:main_algo1} returns a sampled subset $Y \in \mathcal{C}(\bm{o})$ with probability $\mathbb{P}(Y) \le  1/\delta^N \cdot \tilde{\mathbb{P}}(\bm{Y}=Y)$ where $N$ is the number of agents, $\tilde{\mathbb{P}}$ is defined in Eq. \ref{def:q-dpp}, $\delta$ is defined in Assumption \ref{assump:single}. 
\end{theorem}
\begin{proof}
 
  This result can be regarded as a special case of Theorem 3.2 in \citet{celis2018fair} when the number of sample from each partition in $P$-DPP is set to one (please find \emph{Appendix \ref{sec:kdpp_qdpp}} for the differences between $P$-DPP and Q-DPP).  
  
Sine our sampling algorithm generates samples with the probability in proportional to the determinant value $\det (\bm{\mathcal{L}}_Y)$, which is also the nominator in Eq. \ref{def:q-dpp},   
it is then necessary to bound the denominator of the probability of samples from our proposed sampler  so that the error to the exact denominator defined in   Eq. \ref{def:q-dpp} can be controlled.
We start from the Lemma that is going to be used. 
	
	\begin{lemma}[Eckart-Young-Mirsky Theorem]
	\label{lemma:svd}
	For a real matrix $\bm{\mathcal{W}} \in \mathbb{R}^{M \times P}$ with $M \ge P$, suppose that $\bm{\mathcal{W}} = \bm{U} \Sigma \bm{V}^{\top}$ is the singular value decomposition (SVD) of  $\bm{\mathcal{W}}$,  then the best rank $k$ approximation to $\bm{\mathcal{W}}$ under the Frobenius norm $\|\cdot\|_F$ described as  
	$$
	\min_{\bm{\mathcal{W}}': \operatorname{rank}(\bm{\mathcal{W}}') = k} \big\| \bm{\mathcal{W}} - \bm{\mathcal{W}}' \big\|^2_F  
	$$
	is given by  $\bm{{\mathcal{W}}}' = \bm{{\mathcal{W}}}^k=\sum_{i=1}^{k}\sigma _{i}\bm{u}_{i}\bm{v}_{i}^{\top}$ where  $\bm{u}_{i}$ and $\bm{v}_{i}$ denote the $i$-th column of $\bm{U}$ and $\bm{V}$ respectively, and, 
	$${\big \|\bm{\mathcal{W}}-\bm{\mathcal{W}}^k \big\|_{F}^{2}=\Big\|\sum _{i=k+1}^{P}\sigma _{i}\bm{u}_{i}\bm{v}_{i}^{\top }\Big\|_{F}^{2}=\sum _{i=k+1}^{P}\sigma _{i}^{2}}. $$
	Note that the singular values $\sigma_i$ in $\Sigma$ is ranked by size by the SVD procedures such that $\sigma_1 \ge \ldots \ge \sigma_P.$  $\qed$
	\end{lemma}
	
	\begin{lemma}[Lemma 3.1 in \cite{deshpande2006matrix}]
	\label{lemma:svd-k}
	For a matrix $\bm{\mathcal{W}} \in \mathbb{R}^{M \times P}$ with $M \ge P \ge N$, assume $\{\sigma_i\}_{i=1}^{P}$ are the singular values of $\bm{\mathcal{W}}$ and $\bm{\mathcal{W}}_Y$ is the submatrix of $\bm{\mathcal{W}}$ with rows indexed by the elements in  $Y$, then we have 
	$$ \sum_{ |Y|=N} \det \big(\bm{\mathcal{W}}_Y \bm{\mathcal{W}}_Y^{\top}\big) = \sum_{k_1 < \cdots < k_N} \sigma_{k_1}^2 \cdots \sigma_{k_N}^2. \qed$$ 
	\end{lemma}

To stay consistency on notations, we use $N$ for number of agents, $M$ for the size of ground set of Q-DPP, $P$ is the dimension of diverse feature vectors, we assume $M \ge P \ge N$. 
Let $\bm{Y}$ be the random variable representing the output of our proposed sampler in Algorithm \ref{algo:main_algo1}. 
Since the algorithm visit each partition in Q-DPP sequentially, a sample $\tilde{Y}=\big\{(o_1^1, a_1^1),\dots,(o_N^{|\mathcal{O}|}, a_N^{|\mathcal{A}|}) \big\}$  is therefore an ordered set. 
Note that the algorithm is agnostic to the partition number (i.e. the agent identity), without losing generality, we denote the first partition chosen as  $\mathcal{Y}_1$.
We further denote $\tilde{Y}_i, i \in \{1,\ldots,N\}$ as the $i$-th observation-state pair in $\tilde{Y}$, and $\mathcal{I}(\tilde{Y}_i) \in \{1,\ldots,N\}$ denotes the partition number where $i$-th  pair is sampled.

According to the Algorithm \ref{algo:main_algo1}, at first step, we choose $\mathcal{Y}_1$, and based on the corresponding observation $o_1$, we then locate the valid subsets $ \forall (o, a) \in  \mathcal{Y}_i(o_i)$, and finally sample one observation-action pair from the valid set  $\mathcal{Y}_i(o_i)$ with probability proportional to  the norm of the vector defined in the Line $4-5$ in Algorithm \ref{algo:main_algo1}, that is, 
\begin{equation}
 \mathbb{P}(\tilde{Y}_i) \propto	\big\|\bm{w}_{\mathcal{J}(o, a)}\big\|^2= \big\|\bs{b}_{\mathcal{J}(o, a)}\big\|^{2}\exp{\big(\bm{\mathcal{D}}_{\mathcal{J}(o, a), \mathcal{J}(o, a)} \big)}.  
\end{equation}
After $\tilde{Y}_i$ is sampled, the algorithm then moves to the next partition and repeat the same process until all $N$ partitions are covered. 

The specialty of this sampler is that before sampling at each partition $i \in \{1,\ldots,N\}$,  the Gram-Schmidt process will be applied to ensure all the rows in the $i$-th partition of $\bm{\mathcal{W}}$ to be orthogonal to all previous sampled pairs
$$
\bs{b}^i_j=\amalg_{\operatorname{span}\{B^i\}}\left(\bs{b}^{i-1}_j\right), \forall j \in \{1,...,M\} - J .
$$
where  $B^i=\{ \bs{b}^t_{\mathcal{J}(o^t, a^t)}\}_{t=1}^{i-1}, J=\{\mathcal{J}(o^t, a^t) \}_{t=1}^{i-1}$.
Note that since $\bm{\mathcal{D}}$ only contributes a scalar to $\bm{w}_j$,  and $\bs{b}_j$ is a $P$-dimensional vector same as $\bm{w}_j$, in practice, the Gram-Schmidt orthorgonalization  needs only conducting on $\bs{b}_j$ in order to make rows of $\bm{\mathcal{W}}$ mutually orthogonal.

Based on the above sampling process and each time-step $i$, we can write the probability of getting a  sample $\tilde{Y}$ by 
\begin{equation}
\begin{aligned}[b]
\label{eq:sampleprob}
&	\mathbb{P}(\bm{Y}=\tilde{Y}) \\ 
= & \  \mathbb{P}\big(\tilde{Y}_1 \big) \prod_{i=2}^{N} \mathbb{P}\big(\tilde{Y}_i \big| \tilde{Y}_{1},\dots,\tilde{Y}_{i-1}  \big)  \\
=&\prod_{i=1}^{N} \dfrac{\Big\|\amalg_{\operatorname{span}\{B^i\}}\left(\bm{w}_{\mathcal{I}(\tilde{Y}_i)}\right)\Big\|^{2}}{\sum_{(o, a) \in \mathcal{Y}_{\mathcal{I} \left(\tilde{Y}_{i}\right)}}\Big\|\amalg_{\operatorname{span}\{B^i\}}\left(\bm{w}_{\mathcal{I}(o, a)}\right)\Big\|^{2}} \\
=& \dfrac{\prod_{i=1}^{N}\Big(\Big\|\amalg_{\operatorname{span}\{B^i\}}\left(\bm{w}_{\mathcal{I}(\tilde{Y}_i)}\right)\Big\|^{2} \Big)}{\prod_{i=1}^{N}\Big( \sum_{(o, a) \in \mathcal{Y}_{\mathcal{I} \left(\tilde{Y}_{i}\right)}}\Big\|\amalg_{\operatorname{span}\{B^i\}}\left(\bm{w}_{\mathcal{I}(o, a)}\right)\Big\|^{2}\Big)} \\
=& \dfrac{\det \big(\bm{\mathcal{W}}_{\tilde{Y}}\bm{\mathcal{W}}_{\tilde{Y}}^{\top}\big)}{\prod_{i=1}^{N}\Big( \sum_{(o, a) \in \mathcal{Y}_{\mathcal{I} \left(\tilde{Y}_{i}\right)}}\Big\|\amalg_{\operatorname{span}\{B^i\}}\left(\bm{w}_{\mathcal{I}(o, a)}\right)\Big\|^{2}\Big)} \ \ \ , 
	\end{aligned}
\end{equation} 
where the $4$-th equation in Eq.~\ref{eq:sampleprob} is valid because of Proposition \ref{prop1:gramschmidt}. 

For each term in the denominator, according to the definition of the operator $\amalg_{\operatorname{span}\{B^i\}}$,  we can rewrite into
{\small  
\begin{equation}
\begin{aligned}[b]
\sum_{(o, a) \in \mathcal{Y}_{\mathcal{I} \left(\tilde{Y}_{i}\right)}}\Big\|\amalg_{\operatorname{span}\{B^i\}}\left(\bm{w}_{\mathcal{I}(o, a)}\right)\Big\|^{2} = \Big \| \bm{\mathcal{W}}_{{\mathcal{I} (\tilde{Y}_{i})}} - \bm{\mathcal{W}}_{{\mathcal{I} (\tilde{Y}_{i})}}' \Big \|_{F}^2 \nonumber
	\end{aligned}
\end{equation} }
where the rows of $\bm{\mathcal{W}}_{{\mathcal{I} (\tilde{Y}_{i})}}$ are $\{\bm{w}_{\mathcal{I}(o, a)} \}_{(o, a) \in \mathcal{Y}_{\mathcal{I}(\tilde{Y}_i)}}$ which are essentially the submatrix of $\bm{\mathcal{W}}$ that corresponds to partition $\mathcal{I}(\tilde{Y}_i)$, and the rows of $\bm{\mathcal{W}}_{{\mathcal{I} (\tilde{Y}_{i})}}^{'}$ are the orthogonal projections of $\{\bm{w}_{\mathcal{I}(o, a)} \}_{(o, a) \in \mathcal{Y}_{\mathcal{I}}}$ onto $\operatorname{span}\{B^i\}$, and we know $\operatorname{rank}\big(\bm{\mathcal{W}}_{{\mathcal{I} (\tilde{Y}_{i})}}^{'}\big)=|B^i| = i-1$. 
According to Lemma \ref{lemma:svd},  with $\hat{\sigma}_{{\mathcal{I} (\tilde{Y}_{i})}, k}$ being the $k$-th singular value of 
$\bm{\mathcal{W}}_{{\mathcal{I} (\tilde{Y}_{i})}}$, we know that
\begin{equation}
\begin{aligned}[b]
 \Big \| \bm{\mathcal{W}}_{{\mathcal{I} (\tilde{Y}_{i})}} - \bm{\mathcal{W}}_{{\mathcal{I} (\tilde{Y}_{i})}}' \Big \|_{F}^2 \ge \sum_{k=i}^{P} \hat{\sigma}_{{\mathcal{I} (\tilde{Y}_{i})}, k}^2  \ \ .
	\end{aligned}
\end{equation} 

 Therefore, we have the denominator of Eq.~\ref{eq:sampleprob} as:
 \begin{equation}
\begin{aligned}[b]
\label{eq:normalizer}
&\prod_{i=1}^{N}\Big( \sum_{(o, a) \in \mathcal{Y}_{\mathcal{I} \left(\tilde{Y}_{i}\right)}}\Big\|\amalg_{\operatorname{span}\{B^i\}}\left(\bm{w}_{\mathcal{I}(o, a)}\right)\Big\|^{2}\Big)	\\
 \ge & \ \  \prod_{i=1}^{N}\sum_{k=i}^{P} \hat{\sigma}_{{\mathcal{I} (\tilde{Y}_{i})}, k}^2 \ \ 
 \ge \ \  \prod_{i=1}^{N}\sum_{k=i}^{P} \delta \cdot \sigma_k^2  \ \ \ \ \ \ \ \ \  (\text{\emph{Assumption \ref{assump:single}}}) \\
 \ge & \ \ \ \delta^N \cdot \hspace{-5pt} \sum_{k_1 < \cdots < k_N} \sigma_{k_1}^2 \cdots \sigma_{k_N}^2 \\
 = & \  \delta^N \cdot  \sum_{Y \subseteq \mathcal{Y}:  |Y|=N} \det \big(\bm{\mathcal{W}}_Y \bm{\mathcal{W}}_Y^{\top}\big)  \ \ \ \ \ \ \ \ \ \ \ \ \ \ \   (\text{\emph{Lemma \ref{lemma:svd-k}}})   \\
 \ge & \ \ \ \delta^N \cdot  \sum_{Y \in \mathcal{C}(\bm{o})} \det \big(\bm{\mathcal{W}}_Y \bm{\mathcal{W}}_Y^{\top}\big) 
\end{aligned}
\end{equation} 
Taking Eq.~\ref{eq:normalizer} into Eq.~\ref{eq:sampleprob}, we can obtain that 
$$
\mathbb{P}(\bm{Y}=\tilde{Y}) \le  \dfrac{\delta^N \cdot \det \big(\bm{\mathcal{W}}_{\tilde{Y}}\bm{\mathcal{W}}_{\tilde{Y}}^{\top}\big)}{ \sum_{Y \in \mathcal{C}(\bm{o})} \det \big(\bm{\mathcal{W}}_Y \bm{\mathcal{W}}_Y^{\top}\big) } =  1/\delta^N  \cdot \tilde{\mathbb{P}}(\bm{Y}=\tilde{Y})
$$
 where $\mathbb{P}(\bm{Y}=\tilde{Y})$ is the probability of obtaining the sample $\tilde{Y}$ from our proposed sampler and $\tilde{\mathbb{P}}(\bm{Y}=\tilde{Y})$ is the probability of getting that sample $\tilde{Y}$ under Eq. \ref{def:q-dpp}.
 	\end{proof}

\clearpage
\onecolumn

\subsection{Difference between Q-DPP and $P$-DPP}
\label{sec:kdpp_qdpp}
The design of Q-DPP and its samply-by-projection sampling process is inspired by and based on $P$-DPP \cite{celis2018fair}. 
 However, we would like to highlight the multiple differences in that \textbf{1)} $P$-DPP is designed for modeling the fairness for data summarization whereas Q-DPP serves as a function approximator for the joint Q-function in the context of multi-agent learning; \textbf{2)} though we analyze Eq. \ref{eq:sampleprob} based on $\bs{\mathcal{W}}$, the actual orthorgonalziation step of our sampler only needs performing on the vectors of $\bs{b}_j$ rather than the entire matrix $\bs{\mathcal{W}}$ due to our unique quality-diversity decomposition on the joint Q-function in Eq. \ref{eq:qd_decomp}; \textbf{3)} the set of  elements in each partition $\mathcal{Y}_i(o_i)$ of Q-DPP change with the observation at each time-step, while the partitions stay fixed in the case of $P$-DPP;  \textbf{4)} the parameters of  $\bs{\mathcal{W}}$ are learned through a trail-and-error multi-agent reinforcement learning process compared to the cases in $P$-DPP  where the kernel is given by hand-crafted features (e.g. SIFT features on images); \textbf{5)} we implement the constraint in Assumption \ref{assump:single} via a penalty term during the CTDE  learning  process, while $P$-DPP does not consider meeting such assumption through optimization.

\subsection{Time Complexity of Algorithm \ref{algo:main_algo1}}
\label{timecomp}
Let’s analyze the time complexity of the proposed Q-DPP sampler in steps $1-10$ of   Algorithm  \ref{algo:main_algo1}. Given the observation $o$, and the input matrices $\bm{\mathcal{D}}, \bm{\mathcal{B}}$ (whose sizes are $M \times M$, $P \times M$, with $M=|A| \times N$ being the size of all $N$ agents' allowed actions under $o$ and $P$ being the diversity feature dimension), the sampler samples one action for each agent sequentially, so the outer loop of step 3 is $\mathcal{O}(N)$. Within the partition of each agent, step 4 is $\mathcal{O}(P)$, step 5 is $\mathcal{O}(P|A|)$, step 6 is $\mathcal{O}(1)$, so the complexity so far is $\mathcal{O}(NP|A|)$. Computing step 8 for ALL partitions is of $\mathcal{O}(N^2P|A|) ^{\dagger}$. The overall complexity is $\mathcal{O}(N^2P|A|)=\mathcal{O}(NMP)$, since the input is $\mathcal{O}(MP)$ and the agent  number  $N$ is a constant, our sampler has linear-time complexity with respect to the input, also linear-time with respect to the number of agents. Such argument is in line with 
the project-and-sample sampler in \citet{celis2018fair}. 

$^{\dagger}$: In the Gram-Schmidt process, orthogonalizing a vector to another takes $\mathcal{O}(P)$. Considering all valid actions for each agent takes $\mathcal{O}(P|A|)$. Note that while looping over different partitions, the remaining unsampled partitions do not need repeatedly orthogonalizing to all the previous samples, in fact, they only need orthogonalizing to the LATEST sample. In the example of Fig 2, after agent 2 selects action 5, agent 3’s three actions only need orthogonalizing to action 5 but not action 2 because it has been performed when the partition of agent 1 was visited. So the total number of orthogonalization is $(N-1)N/2$ across all partitions, leading to $\mathcal{O}(N^2P|A|)$ time for step 8.

\section{Experimental Parameter Settings}
\label{exp_detail}
\label{sec:qdpp_settings}
The hyper-parameters settings for Q-DPP are given in Table~\ref{tb:hyper}. 
For all experiments we update the target networks after every 100 episodes. All activation functions in hidden layers are ReLU.  The
optimization is conducted using RMSprop with a learning rate of $5 \times 10^{-4}$ and $\alpha=0.99$ with no weight decay or momentum.
\begin{table*}[h]
\vspace{-5pt}
\caption{Q-DPP Hyper-parameter Settings.}
\label{tb:hyper}
\vspace{-13pt}
\begin{center}
\begin{small}
\begin{sc}
\begin{tabular}{lcl}
\toprule
\hspace{15pt}\textbf{Common Settings }& \textbf{Value} & \textbf{Description}  \\
\midrule
Learning Rate& $0.0005$ & Optimizer learning rate.\\
Batch Size &  $32$ & Number of episodes to use for each update. \\
Gamma & $0.99$ &  Long Term Discount factor.\\
Hidden Dimension & $64$ & Size of hidden states. \\
Number of Hidden Layers & $3$ & Number of Hidden layers. \\
target update interval & $100$ & interval of updating the target network. \\ \hline
\multicolumn{3}{l}{\textbf{Multi-Step Matrix Game}}  \\  \hline
Step    & 40K& Maximum time steps.\\
Feature Matrix Size & $176 \times 32$ & Number of observation-action pair times embedding size. \\
Individual Policy Type & RNN & Recurrent DQN. \\
epsilon decay scheme & \multicolumn{2}{l}{linear decay from $1$ to $0.05$ in  $30$K steps.}\\ \hline
\multicolumn{3}{l}{\textbf{Coordinated Navigation}}  \\  \hline
Step    & 100K& Maximum time steps.\\
Feature Matrix Size & $720 \times 32$ & Number of observation-action pair times embedding size. \\
Individual Policy Type & RNN & Recurrent DQN. \\
epsilon decay scheme & \multicolumn{2}{l}{linear decay from $1$ to $0.1$ in  $10$K steps.}\\
\hline
\multicolumn{3}{l}{\textbf{Blocker Game}}  \\  \hline
Step    & 200K& Maximum time steps.\\
Feature Matrix Size & $420 \times 32$ & Number of observation-action pair times embedding size. \\
Individual Policy Type & RNN & Recurrent DQN. \\
epsilon decay scheme & \multicolumn{2}{l}{linear decay from $1$ to $0.01$ in  $100$K steps.}\\
\hline
\multicolumn{3}{l}{\textbf{Predator-Prey World} (Four Predators, Two Preys)}  \\  \hline
Step    & 4M& Maximum time steps.\\
Feature Matrix Size & $3920 \times 32$ & Number of observation-action pair times embedding size. \\
Individual Policy Type & Feedforward & Feedforward DQN. \\
epsilon decay scheme & \multicolumn{2}{l}{linear decay from $1$ to $0.1$ in  $300$K steps.}\\
\bottomrule
\end{tabular}
\end{sc}
\end{small}
\end{center}
\vspace{-10pt}
\end{table*}

If not particularly indicated, all the baselines use common settings as listed in Section~\ref{sec:qdpp_settings}.
 IQL, VDN, QMIX, MAVEN and QTRAN use common individual action-value networks as those used by Q-DPP; each consists of two 32-width hidden layers. 
The specialized parameter settings for each algorithm are provided in Table~\ref{tb:hyper_baselines}:
\begin{table*}[h]
\caption{Hyper-parameter Settings for Baseline Algorithms.}
\label{tb:hyper_baselines}
\vspace{-13pt}
\begin{center}
\begin{small}
\begin{sc}
\begin{tabular}{lcl}
\toprule
\hspace{15pt}\textbf{Settings }& \textbf{Value} & \textbf{Description}  \\
\midrule 
\multicolumn{3}{l}{\textbf{QMIX}}  \\  \hline
Monotone Network Layer    & 2& Layer number of Monotone Network.\\
Monotone Network Size & 64 & Hidden layer size  of Monotone Network. \\
\hline
\multicolumn{3}{l}{\textbf{QTRAN}}  \\  \hline
joint action-value network layer    & 2& Layer number of joint action-value network.\\
joint action-value network Size & 64 & Hidden layer size of joint action-value network. \\
\hline
\multicolumn{3}{l}{\textbf{MAVEN}}  \\  \hline
$z$    & 2 & Noise dimension.\\
$\lambda_{MI}$ & $0.001$ & weight of MI objective. \\
$\lambda_{QL}$ & 1 & weight of QL objective. \\
entropy regularization & $0.001$ & Feedforward DQN. \\
discriminator layer & 1 & Number of discriminator network layer. \\
discriminator size & 32 & Hidden Layer Size of discriminator network. \\
\bottomrule
\end{tabular}
\end{sc}
\end{small}
\end{center}
\end{table*}

\clearpage

\section{Solution for Continuous States: Deep Q-DPP}
\label{deep-qdpp}
Although our proposed Q-DPP serves as a new type of function approximator for the value function in multi-agent reinforcement learning, deep neural networks can also be  seamlessly  applied  on Q-DPP. 
Specifically, one can adopt deep networks to respectively represent the quality and diversity terms in the kernels of Q-DPP to tackle continuous state-action space, and we name such approach Deep Q-DPP.  
In Fig. \ref{fig:dpp}, one can think of Deep Q-DPP as modeling $\bm{\mathcal{D}}$ and $\bm{\mathcal{B}}$ by  neural networks rather than look-up tables. 
An analogy of  Deep Q-DPP to Q-DPP would be Deep Q-learning \cite{mnih2015human} to  Q-learning \cite{watkins1992q}.
As the main motivation of introducing Q-DPP is to eliminate structural constraints and  bespoke neural architecture designs in solving  multi-agent cooperative tasks,  
 we omit  the study of Deep Q-DPP in the main body of this paper. 
Here we  demonstrate a proof of concept for Deep Q-DPP and its effectiveness on StarCraft II micro-management tasks \cite{samvelyan2019starcraft} as an initiative. 
However, we do believe a full treatment needs substantial future work. 

\subsection{Neural Architectures for Deep Q-DPP.}
\begin{figure}[h]
	\centering
	\vspace{0pt}
	\includegraphics[width=1.0\textwidth]{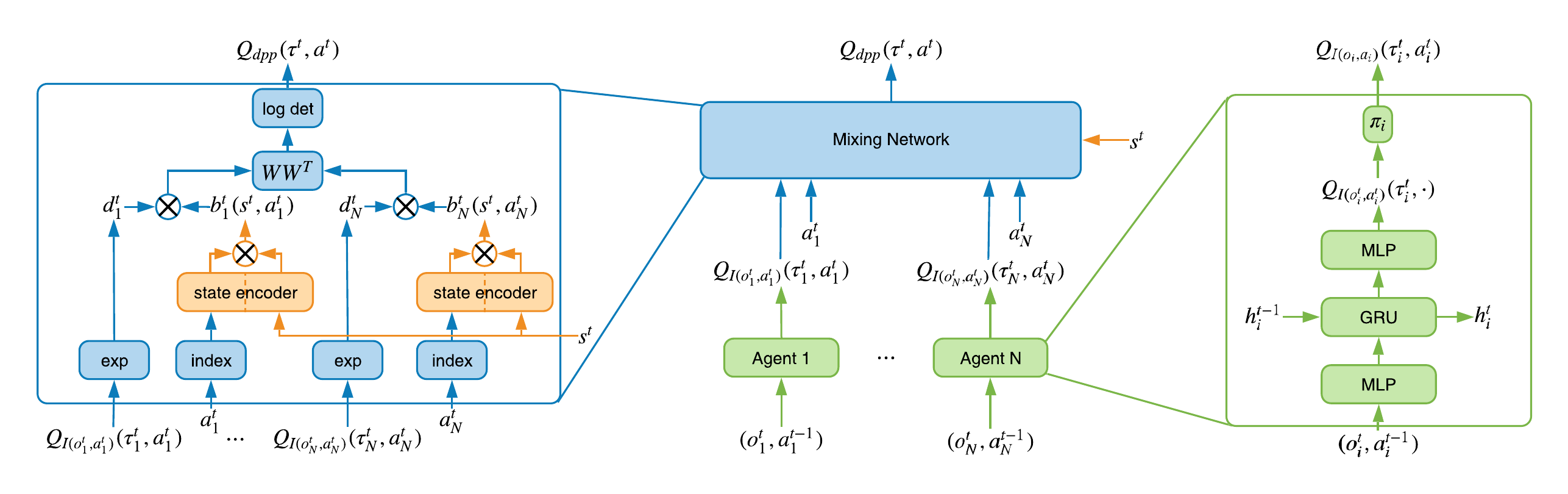}
    \vspace{-20pt}
	\caption{Neural Architecture of Deep Q-DPP. The middle part of the diagram shows the overall architecture of Q-DPP, which consists of each agent's individual Q-networks and a centralized mixing network.
Details of the mixing network are presented in the left.
We compute the quality term, $d_i$, by applying the exponential operator on the individual Q-value, 
and compute the diversity feature term, $\bm{b}_i$, by 
index the corresponding vector in $\bm{\mathcal{B}}$ 
through the global state $s$ and each  action $a_i$. 
	}
	\label{fig:qdpp architecture}
 \vspace{-0pt}
\end{figure}


A critical advantage of Deep Q-DPP is that it can deal with continuous states/observations.  
When the input state $s$  is continuous, we first index the raw diversity feature $\bm{b}'_i$ based on the embedding of discrete action  $a_i$.
To integrate the information of the continuous state, 
we use two  multi-layer feed-forward neural networks $f_d$ and $f_n$, which encodes the direction and norm of the diversity feature separately.
$f_d$ outputs a feature vector with same shape as $\bm{b}'_i$ indicating the \textbf{direction}, and $f_n$ outputs a real value for computing the \textbf{norm}.
In practice, we find modeling the direction and norm of the diversity features separately by two neural networks helps stabilize training, and 
the diversity feature vector is computed as $\bm{b_i}=f_d(\bm{b}'_i, s)\times \sigma(f_n(\bm{b}'_i, s))$.
Finally, the centralized Q-value can then be computed from $d_i$ and $\bm{b_i}$ following Eq. \ref{eq:qd_decomp}.

\newpage
\subsection{Experiments on StarCraft II Micro-Management}
\begin{figure}[h]
	\centering
	\begin{subfigure}[l]{.49\textwidth}
		\centering
		\includegraphics[width=\textwidth]{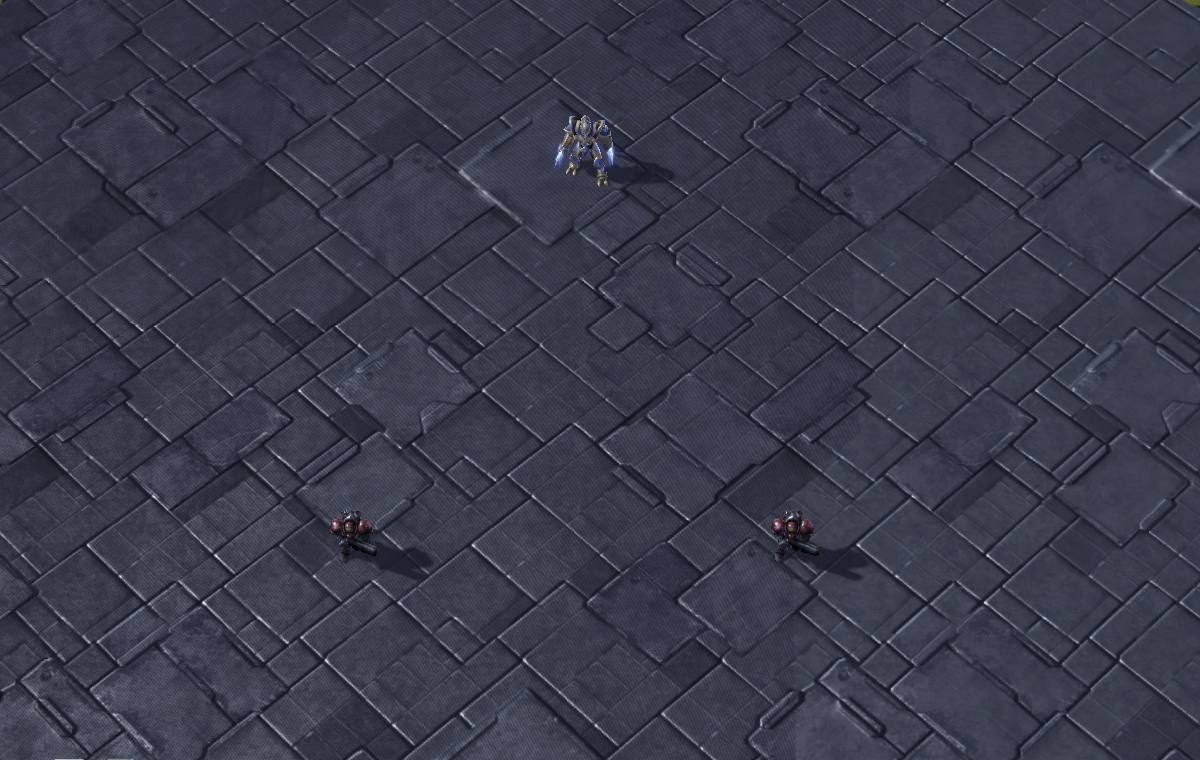}
		\vspace{15pt}
		\caption{Scenario Screenshot}
		\label{fig:3m_game}
	\end{subfigure}
	\begin{subfigure}[l]{.5\textwidth}
		\centering
		\includegraphics[width=\textwidth]{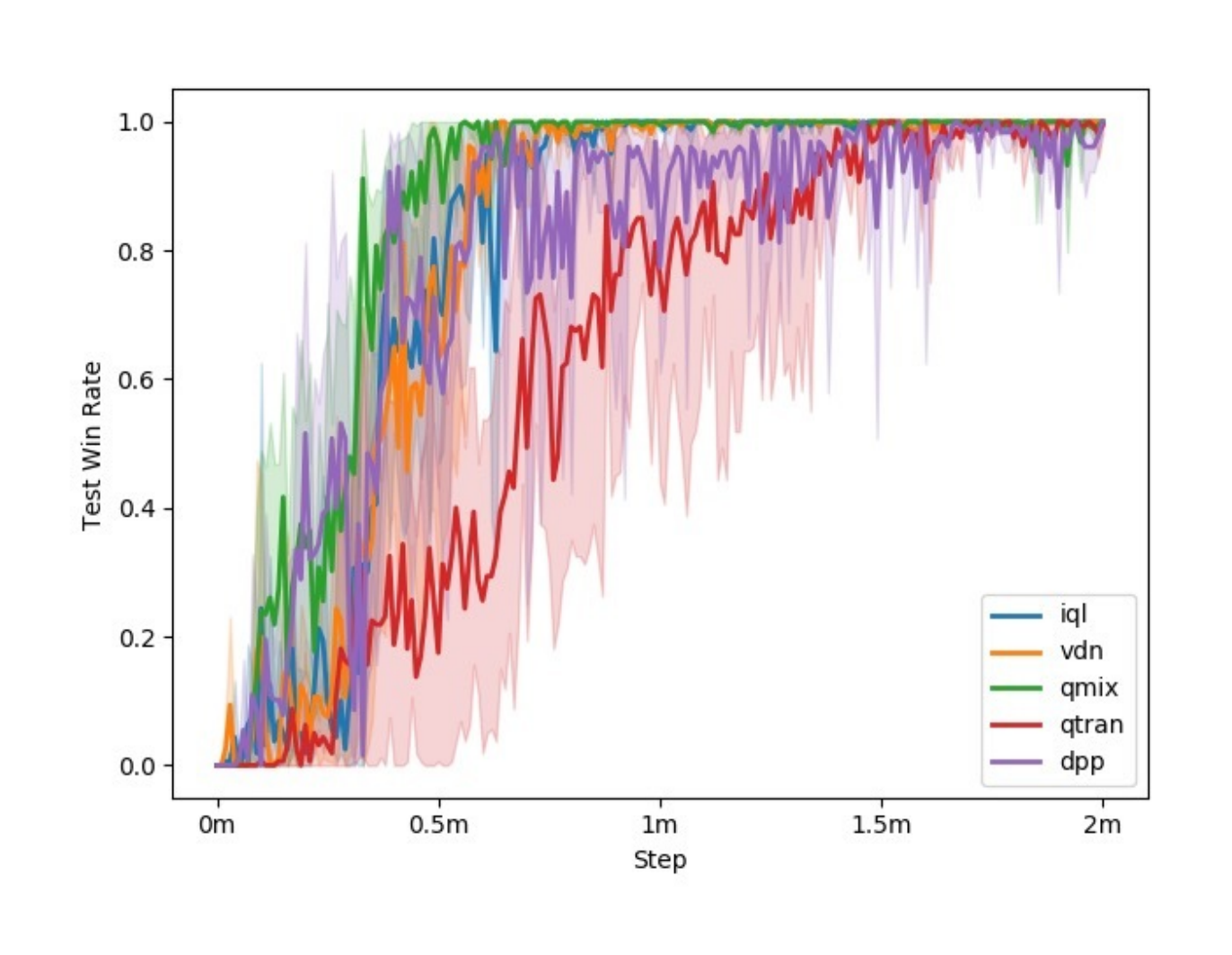}
		\caption{2m\_vs\_1z}
		\label{fig:2m_vs_1z_game}
	\end{subfigure}
	\caption{StarCraft II micro-management on the scenario of  \emph{2 Marines vs. 1 Zealot} and its performance. }
	\label{fig:sc2_games}
\end{figure}

%
%


We study one of the simplest continuous state-action micro-management games in StarCraft II in SMAC \cite{samvelyan2019starcraft}, i.e., \textbf{2m\_vs\_1z}, the screenshots of scenarios are given in Fig.~\ref{fig:sc2_games}(a).
In the 2m\_vs\_1z map, we control a team of 2 Marines to fight with 1 enemy Zergling. 
In this task, it requires the Marine units to take advantage of their larger firing range to defeat over  Zerglings which can only attack local enemies.
The agents can observe a \textbf{continuous} feature vector including the information of health, positions and weapon cooldown of other agents. 
In terms of reward design, we keep the default setting. 
All agents receive a large final reward for winning a battle, at the meantime, they also receive immediate rewards that are proportional to the difference of total damages between the two teams in every time-step. 
We compare Q-DPP with aforementioned baseline models, i.e., COMA, VDN, QMIX, MAVEN, and QTRAN, and plot the results in Fig. \ref{fig:sc2_games}(b).
The results show that Q-DPP can perform as good as the state-of-the-art model, QMIX, even when the state feature is continuous.
However, the performance is not stable and presents high variance. 
We believe full treatments need substantial future work.

\end{document}